\documentclass[letterpaper]{article} 
\usepackage{aaai2026}  
\usepackage{times}  
\usepackage{helvet}  
\usepackage{courier}  
\usepackage[hyphens]{url}  
\usepackage{graphicx} 
\usepackage{natbib}  
\usepackage{caption} 
\usepackage[noend]{algpseudocode}
\usepackage{algorithm}
\usepackage{amssymb, amsfonts, amsmath, amsthm, bm}
\usepackage{booktabs}
\usepackage{multirow}
\usepackage{siunitx}

\newtheorem{theorem}{Theorem}
\newtheorem{lemma}{Lemma}

\newtheorem{remark}{Remark}

\newtheorem{restatedthminner}{Theorem}
\newenvironment{restatedthm}[1]{%
\renewcommand\therestatedthminner{#1}%
\restatedthminner
}{\endrestatedthminner}

\newtheorem{restatedleminner}{Lemma}
\newenvironment{restatedlem}[1]{%
\renewcommand\therestatedleminner{#1}%
\restatedleminner
}{\endrestatedleminner}

\usepackage{newfloat}
\usepackage{listings}
\DeclareCaptionStyle{ruled}{labelfont=normalfont,labelsep=colon,strut=off} 
\floatstyle{ruled}
\newfloat{listing}{tb}{lst}{}
\floatname{listing}{Listing}
\nocopyright

\title{Behaviour Policy Optimization: Provably Lower Variance \\ Return Estimates for Off-Policy Reinforcement Learning}
\author {
    Alexander W. Goodall\textsuperscript{\rm 1},
    Edwin Hamel-De le Court\textsuperscript{\rm 1},
    Francesco Belardinelli\textsuperscript{\rm 1}
}
\affiliations {
    \textsuperscript{\rm 1}Imperial College London\\
    \{a.goodall22, e.hamel-de-le-court, francesco.belardinelli\}@imperial.ac.uk
}

\begin{document}

\frenchspacing

\maketitle

\begin{abstract}
Many reinforcement learning algorithms, particularly those that rely on return estimates for policy improvement, can suffer from poor sample efficiency and training instability due to high-variance return estimates. In this paper we leverage new results from off-policy evaluation; it has recently been shown that well-designed behaviour policies can be used to collect off-policy data for provably lower variance return estimates. This result is surprising as it means collecting data on-policy is not variance optimal. We extend this key insight to the online reinforcement learning setting, where both policy evaluation and improvement are interleaved to learn optimal policies. Off-policy RL has been well studied (e.g., IMPALA), with correct and truncated importance weighted samples for de-biasing and managing variance appropriately. Generally these approaches are concerned with reconciling data collected from multiple workers in parallel, while the policy is updated asynchronously, mismatch between the workers and policy is corrected in a mathematically sound way. Here we consider only one worker -- the behaviour policy, which is used to collect data for policy improvement, with provably lower variance return estimates. In our experiments we extend two policy-gradient methods with this regime, demonstrating better sample efficiency and performance over a diverse set of environments.
\end{abstract}



\begin{links}    
    \link{Code}{https://github.com/sacktock/BPO}
\end{links}

\section{Introduction}

Reinforcement learning (RL) \cite{sutton2018reinforcement} has seen notable success in a diverse range of sequential decision making tasks, including advanced control systems \cite{schulman2017proximalpolicyoptimizationalgorithms}, high dimensional and complex games (e.g., Atari, StarCraft II, Go) \cite{mnih2015human, vinyals2019grandmaster, silver2017mastering}, and fine-tuning for Large Reasoning Models (LRM) \cite{deepseekai2025deepseekr1incentivizingreasoningcapability}. Learning optimal policies in such scenarios might require millions, 
even billions of agent-environment interactions \cite{mnih2015human, mnih2016asynchronous, jaderberg2016reinforcementlearningunsupervisedauxiliary}, which is often slow and expensive in many scenarios \cite{zhang2023new}. Off-policy RL \cite{wang2017sampleefficientactorcriticexperience, espeholt2018impala} aims to improve sample efficiency by reusing past experiences or data generated by a different behaviour policy, rather than following the current policy. However, off-policy methods may introduce new challenges -- the mismatch between the behaviour policy and the target policy can lead to high-variance return estimates and unstable learning \cite{agarwal2017effective}. Even na\"{i}ve off-policy correction with importance sampling (IS) can suffer from variance explosion when the behaviour policy and target policy differ \cite{precup2000eligibility}.

An alternative paradigm is offline RL \cite{ernst2005tree, fujimoto2019off, levine2020offline}, which offers a promising 
direction for learning optimal policies from limited data without having to interact online. However, many practitioners still deploy their policies online, to evaluate the final performance or tune hyperparameters \cite{fu2021d4rl,schrittwieser2021online,mathieu2023alphastar}, essentially violating the core assumption of this paradigm. Thus, off-policy evaluation is a crucial challenge in offline RL \cite{zhang2020gradientdice, liu2024efficient}, where the performance of a policy must be evaluated based on a fixed offline dataset. The Offline Data Informed (ODI) algorithm proposed by \citeauthor{liu2024efficient} (\citeyear{liu2024efficient}), offers a surprising new result -- well-designed behaviour policies can provide provably lower-variance return estimates for more (sample) efficient policy evaluation. 

\paragraph{Contributions.} We propose behaviour policy optimization (BPO), a new variance reduction technique for online RL, using off-policy data collected by a well-designed behaviour policy for provably lower-variance return estimates. 
For designing the behaviour policy we apply new results from off-policy evaluation in \cite{liu2024efficient} to the (discounted) online RL setting. We note that this extension is non-trivial and comes with additional technical challenges, e.g., non-stationarity of the target policy. Furthermore, the full undiscounted Monte Carlo return typically considered in off-policy evaluation is usually a bad choice for policy improvement in online RL \cite{williams1992simple, mnih2016asynchronous}, thus we introduce the truncated IS weighted TD($\lambda$) returns for estimating the target policy returns, we provide a proof of variance reduction and unbiasedness for this new estimator. 

Practically, we build upon and extend two established policy-gradient algorithms: REINFORCE \cite{williams1992simple} and Proximal Policy Optimization (PPO) \cite{schulman2017proximalpolicyoptimizationalgorithms}; both of which use return estimates in some way for policy improvement. Rather than collecting data on-policy, we concurrently optimize a behaviour policy used to collect off-policy data for training the target policy. The behaviour policy is optimized using a different loss function that includes Q-value estimates of the target policy, thus in addition to the behaviour policy, we train two additional Q-networks with Fitted Q-Evaluation (FQE) \cite{le2019batch}. Naturally, this increases the complexity of the full algorithm, and introduces additional challenges with hyperparameter tuning. However, for both REINFORCE and PPO we demonstrate no-worse and often better sample efficiency and final performance with respect to episode returns. Finally, we note that our methodology is not limited to discrete actions, we design an alternative loss function for optimizing the behaviour policy with continuous actions, which has a sound theoretical justification.

\section{Related Work}
\label{sec:relatedwork}

\paragraph{Variance Reduction and Off-policy RL.}

Variance reduction techniques have long been studied in RL to improve learning efficiency and evaluation accuracy. In policy gradient methods, for example, using a baseline or advantage function to reduce the variance of gradient estimates is standard practice \cite{williams1992simple, sutton2018reinforcement}. Generalized Advantage Estimation \cite{schulman2018high} further trades off bias and variance by averaging multi-step returns, greatly stabilizing on-policy learning. Off-policy RL underlines many significant breakthroughs in RL, e.g., training from experience replay \cite{mnih2015human}, scalable and parallelize training, e.g., IMPALA \cite{espeholt2018impala} and (offline) AlphaStar \cite{mathieu2023alphastar}. If the offline dataset has good coverage, optimal policies can be learnt without having to collect data online, which is often more expensive \cite{zhang2023new}. Importance sampling (IS) is the classical technique for correcting policy mismatch \cite{precup2000eligibility}, although unbounded importance weights can lead to high-variance estimates \cite{precup2000eligibility}. To combat this issue \citeauthor{wang2017sampleefficientactorcriticexperience} (\citeyear{wang2017sampleefficientactorcriticexperience}) introduced an off-policy actor-critic with truncated importance weights (ACER), \citeauthor{espeholt2018impala} (\citeyear{espeholt2018impala}) also showed that truncating importance weights for V-trace (IMPALA) reduced variance and yields more stable learning in deep RL, although IMPALA is orthogonal to our research here as it is used to correct for off-policy data collected by multiple workers in parallel with asynchronous policy updates. Furthermore, Retrace($\lambda$) \cite{munos2016safe} is a alternative returns estimator that truncates importance weights to 1, incorporating TD($\lambda$) for bias-variance trade-off, obtaining lower-variance updates (and unbiasedness under certain conditions). 
These approaches generally acknowledge that some bias is acceptable if it dramatically lowers variance and prevents explosive updates. Our work, focused on an orthogonal direction: we actively optimize a behaviour policy that makes the importance weights naturally well-behaved and the data collected most useful without necessarily introducing bias. However, in practice some truncation (bias) is used, although it is typically less aggressive than in related works \cite{munos2016safe, espeholt2018impala} 

\paragraph{Off-policy evaluation.} 
In off-policy evaluation the per-decision IS (PDIS) \cite{precup2000eligibility} returns estimator, (corrected Monte Carlo returns), is commonly used as it is unbiased for the target policy. Unfortunately, the PDIS estimator suffers from high-variance, when the offline data and target policy are poorly aligned, motivating alternatives such as model-based evaluation \cite{mukherjee2022revar} and value-function approximation, e.g., via FQE \cite{le2019batch}. Recent advances in off-policy evaluation include the design of variance-reducing behaviour policies for collecting offline data for off-policy evaluation. Several prior works have explicitly tackled this problem, \citeauthor{hanna2017data} (\citeyear{hanna2017data}) formulated this as a gradient-based search for an optimal behaviour policy to minimize the variance of the IS estimator. They demonstrated improvements in data efficiency, although their approach was limited to (trajectory-level) IS, which is less general than the PDIS estimator and not well-suited to online RL. \citeauthor{mukherjee2022revar} (\citeyear{mukherjee2022revar}) derived a variance-optimal behaviour policy for PDIS, but only under the restrictive assumption of tree-structured MDPs and an accurate estimate the transition dynamics (model-based). Robust On-Policy Sampling (ROS) \cite{zhong2022robust} re-weights an existing behaviour policy toward under-represented states to reduce variance. However, ROS assumes the offline data consists of complete trajectories generated by known policies, and it forgoes IS corrections entirely. In contrast, we leverage ODI \cite{liu2024efficient}, which is a state-of-the-art method for behaviour policy design that operates in general MDPs and with no assumptions on underlying dynamics or data completeness. ODI can leverage (possibly incomplete) arbitrary offline data, and with correct IS weights preserves unbiasedness and comes with theoretical guarantees of variance reduction. This makes ODI much more practical for online RL where we don't necessarily want to wait for complete trajectories before updating our policy. However, all of these methods including ODI, are only concerned with estimating the full undiscounted Monte Carlo returns, which can exhibit high-variance and lead to unstable policy updates in online RL \cite{williams1992simple}, for a more practical implementation we consider an alternative estimator based on TD($\lambda$). Additional challenges arise with online RL, e.g., the underlying distribution generating the ``offline'' dataset is non-stationary, including the policy, which is constantly updated. Preventing over-fitting to past data and quick adaptation to policy updates is also crucial for optimizing the the behaviour policy and Q networks. 

\section{Background}
\label{sec:background}

In this section present the necessary background material, we start with the preliminaries on RL, followed by off-policy corrections with IS ratios. We then describe how importance sampling can be used for variance reduction, and we summarize how variance optimal sampling distributions or behaviour policies can be designed for off-policy evaluation.

\subsection{Preliminaries for RL}

We model the environment as a finite horizon discrete time Markov decision process (MDP) \cite{puterman2014markov}, defined as a tuple $\mathcal{M} := \langle \mathcal{S}, \mathcal{A}, r, p, p_0  \rangle$, where $\mathcal{S}$ is the set of states, $\mathcal{A}$ is the set of actions, $r: \mathcal{S} \times \mathcal{A} \to \mathbb{R}$ is the reward function, $p: \mathcal{S} \times \mathcal{A} \times \mathcal{S} \to [0, 1]$ is the probabilistic transition function, and $p_0:  \mathcal{S} \to [0,1]$ is the initial state distribution.

In this paper we consider with the infinite horizon discounted setting, the discount factor $\gamma \in [0, 1)$ trades-off the relative weighting of short- and long-term rewards.

At time step $t$, an action $A_t$ is sampled according to the agent's policy $\pi: \mathcal{S} \times \mathcal{A} \to [0, 1]$ which defines a probability distribution over actions from the current state $S_t$. The successor state $S_{t+1}$ is then sampled from $p(\cdot \vert S_t, A_t)$. Each episode generates a sequence of states, actions and rewards $S_0A_0R_1S_1A_1R_2S_2\ldots$. The discounted return, which defines the accumulated reward from timestep $t$ onwards, is given by $G_t := \sum^{\infty}_{k=t}\gamma^{k-t}R_{k+1}$. The state- and state-action value functions for the policy $\pi$ are defined as, $v_{\pi}(s) := \mathbb{E}_{\pi}[G_t \mid S_t =s]$ and $q_{\pi}(s, a) := \mathbb{E}_{\pi}[G_t\mid S_t=s, A_t=a]$ respectively. 

In this paper we consider the total discounted return $J(\pi) := \mathbb{E}_{S_0 \sim p_0}[v_{\pi}(S_0)]$ as the performance metric with which to evaluate and optimize the policy. The optimal policy is defined as $\pi^* := \arg\sup_{\pi} J(\pi)$. The most straightforward way to estimate $J(\pi)$ (policy evaluation) is to sample directly with the policy $\pi$ online in the environment. There are many sophisticated algorithms for updating the policy such that $J(\pi)$ is maximized (policy improvement), e.g., value-based \cite{mnih2015human}, policy-based \cite{haarnoja2018soft} and policy-gradient methods \cite{mnih2016asynchronous, schulman2017proximalpolicyoptimizationalgorithms}. In this paper, we are focused on improving the policy evaluation phase, whereby a well-designed behaviour policy $\mu$ is used to collect off-policy data for provably lower-variance return estimates. 

\subsection{Off-policy Corrections in RL}
\label{sec:offpolicycorrections}

The goal in off-policy RL is to use trajectories (sequences) generated by an arbitrary behaviour policy $\mu$, for (optionally) learning the value function $v_{\pi}$ of the target policy $\pi$ and optimizing the target policy $\pi$ to maximize the objective function $J(\pi)$. Since the behaviour policy $\mu$ and the target policy $\pi$ may differ the return estimates must be adjusted to account for the policy mismatch. 

\paragraph{Importance Sampling in RL.} 
Let $\rho_t := \frac{\pi(A_t \vert S_t)}{\mu(A_t \vert S_t)}$ and $\rho_{t: t'}:=\prod^{t'}_{i=t} \frac{\pi(A_i \vert S_i)}{\mu(A_i \vert S_i)}$ denote the per-decision step IS ratio (at time $t$) and the (product) sequence level ratios respectively. The PDIS Monte Carlo returns estimator at time step $t$, $G^{\text{PDIS}}_t$ \cite{precup2000eligibility}, is given by,
\begin{equation}
    G_t^{\text{PDIS}} := \sum^{\infty}_{k=t} \left( \prod^{k}_{i=t} \rho_i \right) \gamma^{k-t} R_{k+1} = \sum^{\infty}_{k=t} \rho_{t:k} \gamma^{k-t} R_{k+1} \label{eq:gpdis}
\end{equation}

For any policy $\mu$ that covers $\pi$ (i.e., $\forall s, a \,\mu(a \vert s) =0 \Rightarrow \pi(a \vert s) =0$), the $G^{\text{PDIS}}$ Monte Carlo returns estimator is unbiased \cite{precup2000eligibility}. Formally, for all $s \in \mathcal{S}$ we have, $\mathbb{E}_{\mu} [ G_t^{\text{PDIS}} \mid S_t = s] = v_{\pi}(s)$. To reduce the variance of (\ref{eq:gpdis}) it suffices to design a specific behaviour policy $\mu$ that produces lower-variance return estimates, while still covering $\pi$ \cite{liu2024efficient}.

\subsection{Variance reduction}
\label{sec:varianceoptimal}

Variance reduction is an important tool in statistics, sampling can either be costly or time consuming; if the variance of our estimator is lower then we require fewer samples to get an accurate estimate \cite{owen2013monte, rubinstein2016simulation}. In RL the problem is two-fold, not only is collecting samples both costly and time consuming \cite{espeholt2018impala, zhang2023new}, but higher-variance estimates lead to higher-variance gradients and more unstable learning \cite{williams1992simple, mnih2016asynchronous, schulman2018high}. First, we provide an intuitive example of how importance sampling can be used to reduce variance.

\paragraph{Example. } Consider the problem of estimating the probability that the random variable $X \sim p(\cdot)$ (with $p(\cdot) = \mathcal{N}(0, 1)$, i.e., standard normal) takes a value greater than $4$, this can be written as $\mathbb{E}_{X \sim \mathcal{N}(0, 1)}[\mathbf{1}(X > 4)]$, where $\mathbf{1}(\cdot)$ denotes the indicator function. After drawing $N$ samples $X_1, X_2, \ldots X_N$, according to $p(\cdot)$, the Monte Carlo estimate for this quantity is given by $I^{(N)} := \sum_{i=1}^N \mathbf{1}[X_i > 4]$. This estimate is unbiased (i.e., $\mathbb{E}[I] = \mathbb{E}_{X \sim \mathcal{N}(0, 1)}[\mathbf{1}[X > 4]]$), although it has high-variance; $X$ taking a value greater than $4$ is a low-probability event. Sampling with a longer tailed distribution, e.g., $q(\cdot) = \mathcal{N}(0, 4)$, is more efficient (lower-variance), correcting for the bias via IS weights $\rho(X) := p(X)/q(X)$, then $I_{IS}^{(N)} := \sum_{i=1}^N \rho(X_i)\mathbf{1}[X_i > 4]$ is the new estimate, for further intuition see Figure \ref{fig:samplingdenisty}.

\begin{figure}[t]
\begin{center}
\includegraphics[width=\columnwidth]{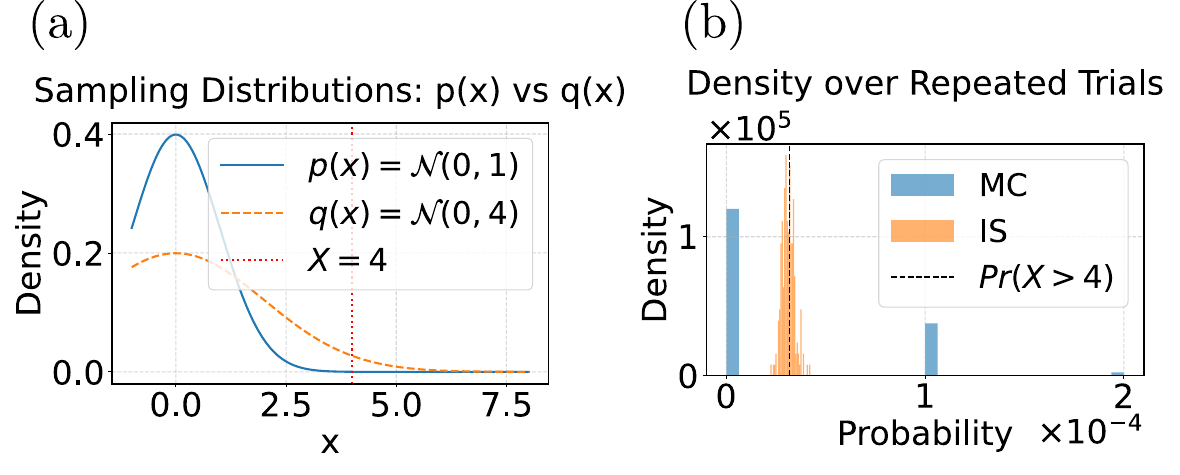}
\caption{(a) Sampling densities. (b) Distribution of estimates over 100,000 trials.}
\label{fig:samplingdenisty}
\end{center}
\end{figure}

\paragraph{Variance Optimal Sampling.} Optimal sampling distributions have been well-studied in statistics \cite{owen2013monte, rubinstein2016simulation}.  Consider an arbitrary distribution $p : \mathcal{X} \to [0, 1]$ over the set $\mathcal{X}$ and a function $f: \mathcal{X} \to \mathbb{R}$. We want to estimate $\mathbb{E}_{X \sim p}[f(X)]$, the goal is to find the optimal sampling distribution $q: \mathcal{X} \to [0, 1]$ such that the variance of the IS estimator, $I_{\text{IS}}=\rho(X)f(x)$ is as small as possible, where $\rho(X) := p(X)/q(X)$ denotes the correct IS weights. Consider the following set of distributions $\Lambda$, 
\begin{equation*}
    \Lambda := \{ q \in \Delta(\mathcal{X}) \mid \forall x \in \mathcal{X} \; q(x)=0 \Rightarrow p(x)f(x) =0  \} \label{eq:enlarged}
\end{equation*}
where $\Delta(\mathcal{X})$ denotes the set of all distributions over $\mathcal{X}$. We note, this is slightly weaker than requiring $q$ to cover $p$, although the estimator remains unbiased \cite{owen2013monte}. Let $\mathbb{V}$ denote the variance, the optimization problem then becomes,
\begin{equation}
    \min_{q \in \Lambda} \mathbb{V}_{X \sim q} [\rho(X)f(X)] \label{eq:problem}
\end{equation}
\citeauthor{owen2013monte} (\citeyear{owen2013monte}) give an optimal solution to (\ref{eq:problem}).
\begin{lemma}[\citeauthor{owen2013monte} \citeyear{owen2013monte}, Ch.~9.1]
\label{lem:optimal}
$\forall x \in \mathcal{X}$ let $q^*(x) \propto p(x)\vert f(x)\vert$. Then $q^*$ is an optimal solution to (\ref{eq:problem}). 
\end{lemma}
\noindent We note here that $q^*(x) \propto p(x)\vert f(x)\vert$ is defined as,
\begin{equation*}
    q^*(x) = p(x)\vert f(x)\vert / \textstyle{\int_{x'\in \mathcal{X}} p(x')\vert f(x')\vert}
\end{equation*}
When $\mathcal{X}$ is discrete, this integral reduces to a sum and $q^*$ can be computed in closed form. However, we note, when $\mathcal{X}$ is continuous this integral may not have a closed form solution unless $f$ has a convenient expression. The following result gives some more intuition on the optimality of $q^*$.
\begin{lemma}[\citeauthor{owen2013monte} \citeyear{owen2013monte}, Ch.~9.1] \label{lem:zero}
If $\forall x \in \mathcal{X}, f(x) \geq 0$ or $\forall x \in \mathcal{X}, f(x) \leq 0$, then $\Lambda = \Lambda_{+}$ (where $\Lambda_{+}$ denotes the set of all distributions giving unbiased estimators), and $q^*$ (defined in Lemma \ref{lem:optimal}) gives zero variance, i.e., $\mathbb{V}_{q^*}(I_{\text{IS}}) =0$. 
\end{lemma}

\paragraph{Variance Optimal Sampling for off-policy evaluation.} We now summarise \citeauthor{liu2024efficient}'s results. \citeauthor{liu2024efficient} (\citeyear{liu2024efficient}) only consider undiscounted finite horizon policy evaluation, thus the notation here is slightly adapted so that it fits into the discounted setting. The goal is to minimize $\mathbb{V}_{\mu}(G^{\text{PDIS}}_0)$ by designing a behaviour policy $\mu$ tailored to a specific target policy $\pi$, such that $\mathbb{E}_{\mu}[G^{\text{PDIS}}_0]$ is still unbiased. \citeauthor{liu2024efficient} (\citeyear{liu2024efficient}) restrict themselves to searching for the \emph{one-step optimal behaviour policy} rather than seeking \emph{global optimality}, which is requires an estimate of the transition probabilities \cite{liu2024efficient} (model-based). The search space of policies is denoted $\Lambda$ (similar to before), which is contained in the space of all unbiased policies, denoted $\Lambda_{+}$, and contains the space of all policies covering $\pi$, denoted $\Lambda_{-}$. The one-step optimal policy, $\hat \mu$, is constructed by assuming we sample with $\hat\mu$ from the current step and $\pi$ thereafter, rather than behaving optimally with $\mu^*$ thereafter, where $\mu^*$ denotes the globally variance optimal policy in $\Lambda$. The optimization problem can be written as,
\begin{equation}
    \min_{\mu \in \Lambda} \mathbb{V}_{A_t \sim \mu, A_{t+1} \sim \pi,\ldots}\big( G^{\text{PDIS}}_t \mid S_t = s \big) \label{eq:onestepoptimal}
\end{equation}
First we define,
\begin{multline}
    \hat q_{\pi}(s, a) = q^2_{\pi}(s, a) + \nu_{\pi}(s, a) + \\\gamma^2 \mathbb{E}_{s' \sim p}\Big[ \mathbb{V}_{\pi}\big(G_{t+1}^{\text{PDIS}} \mid S_{t+1} = s'\big) \Big]
\end{multline}
where $\nu(s, a) = \mathbb{V}_{S_{t+1} \sim p}(\gamma v_{\pi}(S_{t+1}) \mid S_t = s, A_t = a)$, is the next step variance of the value function $v_{\pi}$. Using Lemma \ref{lem:optimal}, \citeauthor{liu2024efficient} (\citeyear{liu2024efficient}) show that the one-step optimal policy $\hat \mu$ satisfying (\ref{eq:onestepoptimal}) is given by,
\begin{equation}
    \hat\mu(a \vert s) \propto \pi(a \vert s) \textstyle{\sqrt{\hat q_{\pi}(s, a)}} \label{eq:propto}
\end{equation}
While $\hat q_{\pi}$ might have a complicated definition, \citeauthor{liu2024efficient} provide a convenient form which we can be learnt. 

\section{Off-policy TD($\lambda$)}
\label{sec:eligibility}

In this section we introduce a new new off-policy returns estimator (based on TD($\lambda$)) which gives us better control of the bias-variance trade-off, and immediately fits into algorithms like PPO \cite{schulman2017proximalpolicyoptimizationalgorithms}. For this estimator we provide a proof of unbiasedness and variance reduction under behaviour policy $\hat \mu$ designed in (\ref{eq:propto}), we also provide a practical way of estimating $\hat q_{\pi}$ in the discounted setting. 

\paragraph{Truncated IS-TD($\lambda$) returns.} We introduce the truncated IS weighted TD($\lambda$) returns, $G^{\text{TIS}, \lambda}$, formally,
\begin{equation}
    G^{\text{TIS}, \lambda}_t = v_{\pi}(S_t) + \sum^{\infty}_{k=t} (\gamma\lambda)^{k-t} \left(\prod_{i=t}^{k-1} c_i \right) \delta_k \label{eq:tislamda}
\end{equation}
where $c_t := \min(\bar c, \frac{\pi(A_t \vert S_t)}{\mu(A_t \vert S_t)})$ and $\rho_t := \min(\bar \rho, \frac{\pi(A_t \vert S_t)}{\mu(A_t \vert S_t)})$ are the truncated IS ratios, and $\delta_k := \rho_k (R_{k+1} + \gamma v_{\pi}(S_{k+1})- v_{\pi}(S_k))$ is the temporal difference (TD) error. 
\begin{remark} \label{rem:recursion}
Similar to eligibility traces \cite{sutton2018reinforcement}, the truncated IS weighted TD($\lambda$) returns can be computed efficiently, by stepping backwards through the following recursive definition, 
\begin{equation*}
    G_t^{\text{TIS}, \lambda} = v_{\pi}(S_t)  + \delta_t + \gamma\lambda c_t\big( G^{\text{TIS, $\lambda$}}_{t+1} - v_{\pi}(S_{t+1})\big) 
\end{equation*} 
\end{remark}
This estimator is reminiscent of V-trace \cite{espeholt2018impala} which also introduces truncated IS ratios, used to control the variance of the estimator. The truncation parameters $\bar c$ and $\bar \rho$ play different roles: since $\rho_t$ appears in the definitions of the TD error, $\bar\rho$ affects the fixed point of convergence of the value function trained with targets $G^{\text{TIS}, \lambda}$; for $\bar \rho$ this fixed point is exactly $v_{\pi}$, for $\bar \rho < \infty$ this fixed point is somewhere between $v_{\pi}$ and $v_{\mu}$. Here $\bar c$ does not affect the fixed point of convergence, the $c_t$s are ``trace cutting'' coefficients used as a variance reduction technique, so that TD errors further in the sequence may play less of a role in the update of the value function, reducing the chance of possible variance explosions \cite{espeholt2018impala}. Furthermore, the parameter $\lambda$ also acts as a variance reduction tool directly trading off bias and variance (with $\lambda=1 \Rightarrow$ low-bias, $\lambda=0 \Rightarrow$ low-variance). We summarize an important property of $G^{\text{TIS}, \lambda}$.
\begin{lemma} \label{lem:vtrace} Let $\bar c, \bar \rho \geq 1$, $\mu = \pi$, then, $\mathbb{E}_{\mu}[G_t^{\text{TIS}, \lambda}] = \mathbb{E}_{\pi}[G_t]$.
\end{lemma}
The analysis can be found in Appendix \ref{sec:lemvtrace}. This property is important, as with appropriate hyperparameter settings we can achieve unbiasedness. We treat $\bar c$ and $\bar \rho$ as hyperparameters; generally this result indicates that a good approach is to set $\bar \rho \geq \bar c \geq 1$. Additionally, $\lambda$ should not stray too far from 1 as it biases our estimate when we rely on value estimates. We continue with the the following result.
\begin{theorem}[Unbiasedness] \label{thm:unbiasedness}
Let $\bar c, \bar \rho = \infty$; then for all $\mu \in \Lambda$, $s \in \mathcal{S}$, $\lambda \in [0,1]$, we have,
    $\mathbb{E}_{\mu}[ G_t^{\text{TIS}, \lambda} \mid S_t = s ] = v_{\pi}(s)$
\end{theorem}
Theorem \ref{thm:unbiasedness} ensures that while we search over the space of policies $\Lambda$ our estimator remains unbiased (under certain conditions). The proof can be found in Appendix \ref{sec:thmunbiasedness}. Now we provide provable variance reduction for $G_t^{\text{TIS}}$ under the behaviour policy $\hat \mu$ designed in (\ref{eq:propto}).

\begin{theorem}[Variance Reduction] \label{thm:variancereduction} 
For any $s \in \mathcal{S}$, $\gamma \in [0,1)$, and $\lambda = 1$, $\bar c, \bar \rho = \infty$,
    \begin{equation*}
        \mathbb{V}_{\hat \mu}\big(G_t^{\text{TIS}, \lambda} \mid S_t =s \big) \leq  \mathbb{V}_{\pi}\big(G_t^{\text{TIS}, \lambda} \mid S_t = s\big) - \epsilon(s)
    \end{equation*}
Where,
\begin{equation*}
    c(s) := \textstyle{\mathbb{E}_{a \sim \pi}[\hat q_{\pi}(s, a)] - \Big(\mathbb{E}_{a\sim \pi}\big[\sqrt{\hat q_{\pi}(s, a)}\big]\Big)^2}
\end{equation*}
And, $\epsilon(s) := c(s) + \gamma^2\mathbb{E}_{A_t \sim \hat \mu_t} \Big[ \rho_t^2 \mathbb{E}_{S_{t+1} \sim p}\big[\epsilon(S_{t+1})\big] \Big]$ 
\end{theorem}
The proof can be found in Appendix \ref{sec:thmvariancereduction}. Noting here that $c(s)$ is always non-negative by Jensen's inequality, which guarantees non-negativity of $\epsilon(s)$ and thus the property of variance reduction. $c(s)=0$ occurs only in the degenerate case when all actions have the same $\hat q_{\pi}$ for state $s$. It remains to show how $\hat q_{\pi}$ can be learnt,
\begin{theorem} \label{thm:qhat}
Let,
\begin{equation}
    \hat r_{\pi}(s,a) := 2 r(s, a) q_{\pi}(s, a) - r^2(s,a)\label{eq:rhatdef}
\end{equation}
Then for any $\gamma \in [0, 1)$,
\begin{equation}
    \hat q_{\pi}(s, a) := \hat r_{\pi}(s, a) + \gamma^2\mathbb{E}_{A_t\sim\pi, S_{t+1}\sim p}[\hat q_{\pi}(S_{t+1}, A_t)] \label{eq:qhatdef}
\end{equation} 
\end{theorem}
The proof here is a straightforward adaptation of \cite{liu2024efficient} and intermediate results from the other theorems (c.f., Appendix \ref{sec:thmqhat}). Theorem \ref{thm:qhat} demonstrates that $\hat q_{\pi}$ corresponds to the state-action value function of the policy $\pi$ but for a different reward function $\hat r$ (c.f., (\ref{eq:rhatdef})). Thus, designing a lower-variance behaviour policy boils down to learning an additional Q-value network to estimate $\hat q_{\pi}$ (also $q_{\pi}$ -- if not already available) and optimizing the behaviour policy $\mu$ so that it matches (\ref{eq:propto}).

\section{Implementation}
\label{sec:implementation}

We now provide a practical implementation of variance reduction via BPO. We implement BPO on top of two policy-gradient algorithms: REINFORCE \cite{williams1992simple} and PPO \cite{schulman2017proximalpolicyoptimizationalgorithms}. In both cases the underlying algorithm remains unchanged, we simply introduce three additional features: (i) two additional Q-value networks are updated with several epochs of batch updates (via FQE \cite{le2019batch}) to provide estimates for $q_{\pi}$ and $\hat q_{\pi}$; (ii) the behaviour policy is updated with several epochs of batch updates to match the target distribution in (\ref{eq:propto}); (iii) the behaviour policy is then used to collect rollouts and the return estimates are computed with (\ref{eq:tislamda}) -- correctly accounting for the mismatch between the behaviour policy and target policy, and truncating IS weights to reduce variance. Algorithm \ref{alg:ppofqe} outlines a general instantiation of BPO implemented on top of a generic policy-gradient algorithm. We continue this section by describing the important implementation details. 

\subsection{Integration with Policy Gradient}
Policy gradient methods directly optimize the policy given as a parametrized function $\pi_{\theta}$, where $\theta$ denotes the policy parameters. The policy is updated with gradient steps using the policy gradient \cite{sutton2018reinforcement}:
\begin{equation}
    \nabla J(\theta) \propto \mathbb{E}_{\pi_{\theta}} \big[ G_t \nabla\ln{\pi_{\theta}(A_t \vert S_t)} \big] \label{eq:policygradient}
\end{equation}

In PPO \cite{schulman2017proximalpolicyoptimizationalgorithms} an additional parametrized value function $V_{\omega}$ is learnt, where $\omega$ denotes the parameters of the value function, and the policy parameters are updated using the clipped surrogate objective function,
\begin{equation}
    L^{\text{clip}}(\theta) = \hat{\mathbb{E}}_{t} \big[\min (r_t(\theta)\hat A_t, \text{clip}( r_t(\theta), 1-\epsilon, 1+ \epsilon)\hat A_t)\big]
    \label{eq:clippedobjective}
\end{equation}
where $\hat{\mathbb{E}}$ corresponds to the empirical expectation under the rollouts. The ratios $r_t(\theta)$ are calculated differently here as $\mu$ is used to collect rollouts (c.f., line 3 in Algorithm \ref{alg:ppofqe}), that is, $r_t(\theta) = \frac{\pi_{\theta}(A_t \vert S_t)}{\mu(A_t \vert S_t)}$ corresponds to the ratio between the current policy $\pi_{\theta}$ and the behaviour policy $\mu$. Furthermore, the advantage estimates are modified to use the truncated IS TD($\lambda$) returns, 
\begin{equation}
    \hat A_t := \hat G^{\text{TIS}, \lambda}_t - V_{\omega}(S_t)
\end{equation}
where $\hat G^{\text{TIS}, \lambda}_t$ denotes the truncated IS TD($\lambda$) returns (c.f., \ref{eq:tislamda}), but with value estimates given by $V_{\omega}$. 

The value function $V_{\omega}$ is updated in tandem to minimize the squared error between its predictions and the truncated IS TD($\lambda$) estimates,
\begin{equation}
    L^{\text{value}}(\omega) = \frac{1}{2}\hat{\mathbb{E}}_t\Big[ \big(V_{\omega}(S_t) - \hat G^{\text{TIS}, \lambda}_t \big)^2 \Big] \label{eq:valueloss}
\end{equation}
These objectives are typically optimized jointly with an additional entropy bonus, $L^{\text{joint}}(\theta, \omega)= L^{\text{clip}}(\theta) + \beta_{\text{value}}\cdot L^{\text{value}}(\omega) + \beta_{\text{ent}} \cdot L^{\text{ent}}(\theta)$, where $\beta_{\text{value}}$ and $\beta_{\text{ent}}$ are coefficients providing the relative weightings for each objective. 

\begin{algorithm}[t]
\caption{BPO}
    \label{alg:ppofqe}
\begin{algorithmic}[1]
\State \textbf{Initialize} $\theta$, $\omega$, $\zeta$, $\hat \zeta$, $\xi$ and empty replay buffer $D$
\For{$\text{phase}= 0, 1, \ldots$}
\State Generate $S_0A_0R_1\ldots S_{T-1}A_{T-1}R_{T-1}$ with $\mu$
\State Add experience tuples $(s_i, a_i, r_i, s'_{i})$ to $D$
\State \emph{// Policy gradient update}
\Statex \quad \dots \dots
\Statex \quad \dots \dots
\Statex \quad \dots \dots
\State \emph{// Auxilliary steps}
\For{$\text{epoch} = 0,1, \ldots \text{n\_qvf\_epochs}$}
 \State Optimize $Q_{\zeta}$ with FQE (c.f., (\ref{eq:qlossfn}))
\EndFor
\State Compute $\hat r_t$ (c.f., (\ref{eq:rhatdef}))
\For{$\text{epoch} = 0,1, \ldots \text{n\_qvf\_epochs}$}
    \State Optimize $\hat Q_{\hat\zeta}$ with FQE (c.f., (\ref{eq:qhatlossfn}))
\EndFor
\For{$\text{epoch} = 0,1, \ldots \text{n\_mu\_epochs}$}
    \State Optimize behaviour policy $\mu_{\xi}$ (c.f., (\ref{eq:mudiscloss}) or (\ref{eq:mucontloss}))
\EndFor
\EndFor
\end{algorithmic}
\end{algorithm}

\subsection{Fitted Q-evaluation}

Recall that to design variance optimal behaviour policies we need to estimate the state-action value function $\hat q_{\pi}$, which includes the term $q_{\pi}$ in its definition (c.f., \ref{eq:qhatdef}). Thus we train two additional parametrized functions, $Q_{\zeta}$ and $\hat Q_{\hat \zeta}$.

For training both $Q_{\zeta}$ and $\hat Q_{\hat \zeta}$ we use FQE \cite{le2019batch} -- a method for learning the state-action value functions of arbitrary policies from an offline dataset. In our instance, the offline dataset corresponds to a relatively small replay buffer $D$ containing tuples $\{(s_i, a_i, r_i, s'_i)\}^{m}_{i=1}$ of the most recent rollout data collected by $\mu$. The two Q-networks, $Q_{\zeta}$ and $\hat Q_{\hat \zeta}$ are updated for several epochs using one-step Bellman targets computed from batches of data sampled from the replay buffer $D$. 

First $Q_{\zeta}$ is updated to estimate $q_{\pi}$, using the following loss function,
\begin{equation}
    L^{FQE}(\zeta) = \frac{1}{2} \hat{\mathbb{E}}_t \Big[  \big( Q_{\zeta}(S_t, A_t) - Q^{\text{targ}}_t \big)^2\Big] \label{eq:qlossfn}
\end{equation}
where $Q^{\text{targ}}_t = r_t + \gamma \mathbb{E}_{A_{t+1} \sim\pi_{\theta}} \big[Q_{\zeta}(s'_t, A_{t+1}) \big]$ corresponds to the one-step Bellman targets. Once $Q_{\zeta}$ is trained then the rewards $\hat r_t$ are calculated with (\ref{eq:rhatdef}) and the tuples are updated, i.e., $D=\{(s_i, a_i, \hat r_i, s'_i)\}^{m}_{i=1}$. The loss function for $\hat Q_{\hat \zeta}$ is,
\begin{equation}
    \hat L^{FQE}(\hat \zeta) = \frac{1}{2} \hat{\mathbb{E}}_t \Big[  \big( \hat Q_{\hat \zeta}(S_t, A_t) - \hat Q^{\text{targ}}_t \big)^2\Big] \label{eq:qhatlossfn}
\end{equation}
except $r_t$, $Q_{\zeta}$ and $\gamma$, are replaced by $\hat r_t$, $\hat Q_{\hat \zeta}$ and $\gamma^2$ for computing the targets, $\hat Q^{\text{targ}}_t = \hat r_t + \gamma^2 \mathbb{E}_{A_{t+1} \sim\pi_{\theta}} \big[ \hat Q_{\zeta}(s'_t, A_{t+1}) \big]$.

\paragraph{Stabilizing learning.} The stability of the estimates for $q_{\pi}$ and $\hat q_{\pi}$ are critical to BPO. Value-based approaches, including FQE, are known to be prone to overestimation and approximation bias \cite{fujimoto2022should}. To stabilize learning we use \emph{symlog targets} \cite{hafner2024mastering}, where the targets $Q^{\text{targ}}_t$ and $\hat Q^{\text{targ}}_t$ are transformed with the symmetric log: $\text{symlog}(x) =\text{sign}(x)\ln(\vert x\vert+1)$. To preserve their magnitudes the predictions are transformed back with the symmetric exp: $\text{symexp}(x) = \text{sign}(x)( \exp(\vert x \vert) - 1)$. Symlog targets squash large magnitudes allowing the two networks to quickly move their predictions to large values when needed. This is crucial as the relative magnitude between $Q_{\zeta}$ and $\hat Q_{\hat \zeta}$ predictions can be large and symlog targets allow us to use the same hyperparameters (e.g., learning rate) for both $Q_{\zeta}$ and $\hat Q_{\hat \zeta}$ and keep the learning dynamics consistent for both networks. Further techniques used to deal with overestimation and stability issues are detailed in Appendix \ref{sec:additionalimplementation}.

\subsection{Behaviour Policy}

The behaviour policy is also a parametrized function $\mu_{\xi}$ with parameters $\xi$. Generally we keep the architecture between $\pi_{\theta}$ and $\mu_{\xi}$ consistent, and initialize $\theta$ and $\xi$ with the same seed, so they are identical at the start of training, although they will likely diverge after just one gradient step.  

\paragraph{Discrete action spaces. } For the discrete case, the target distribution (c.f., (\ref{eq:propto})) can be computed analytically, thus we could sample directly proportional to (\ref{eq:propto}), however to alleviate any possible issues arising from overestimation we still train a separate parametrized behaviour policy $\mu_{\xi}$ using the cross entropy loss to the target distribution:
\begin{equation}
    L_{\mu}^{\text{disc}}(\xi) = -\hat{\mathbb{E}}_t \Big[ \sum_{a \in \mathcal{A}} \mu_{\xi}(a \vert S_t) \ln q (a \vert S_t)  \Big] \label{eq:mudiscloss}
\end{equation}
$q(a \vert s) = \pi_{\theta}(a \vert s) \sqrt{\hat Q_{\hat \zeta} (s, a)} / \sum_{a' \in A} \pi_{\theta}(a' \vert s) \sqrt{\hat Q_{\hat \zeta} (s, a')}$ is the target distribution computed analytically.

\paragraph{Continuous action spaces. } For the continuous case, the target distribution (c.f., (\ref{eq:propto})) cannot be computed analytically without making distributional assumptions on $\hat q_{\pi}$. We still seek to minimize the log probability distance between $\mu(a \vert s)$ and $\pi(a \vert s) \textstyle{\sqrt{\hat q_{\pi}(s, a)}}$ (c.f., (\ref{eq:propto})). We design the following loss function,
\begin{multline}
    L_{\mu}^{\text{cont}}(\xi) = \hat{\mathbb{E}}_t \Big[ \ln \mu_{\xi}(A_t \vert S_t) - \ln \pi_{\theta}(A_t \vert S_t) \\ - \frac{1}{2} \ln \hat Q_{\hat \zeta}(S_t, A_t) \Big] \label{eq:mucontloss}
\end{multline}
We establish the following result for this loss function, the proof can be found Appendix \ref{sec:thmpolicymatching}.
\begin{theorem} \label{thm:policymatching} If $\mu_{\xi}(a \vert s) \propto \pi_{\theta}(a \vert s) \textstyle{\sqrt{\hat Q_{\hat \zeta}(s, a)}}$ (i.e., matches (\ref{eq:propto})), then the loss $L^{\text{cont}}_{\mu}(\xi)$ is minimized. 
\end{theorem}

\section{Experimental Results}
\label{sec:results}

In this section we present our experimental results. We start with REINFORCE \cite{williams1992simple}, which is a simple algorithm that uses the full Monte Carlo returns (i.e., $\lambda=1$) to update the policy with the policy gradient (c.f., (\ref{eq:policygradient})). Given the relatively simplicity of REINFORCE we test our approach on a simple, yet non-trivial toy example from \cite{sutton2018reinforcement}: the \texttt{ShortCorridor} grid world. Given the example is simple, the gains here are marginal, and these results are simply to demonstrate that we can use our method for a variety of policy-gradient algorithms. We then continue by presenting our results for BPO on much more complex environments, where the underlying algorithm we build on top of PPO \cite{schulman2017proximalpolicyoptimizationalgorithms} -- a more sophisticated and widely used algorithm. Additional details of both the REINFORCE and PPO update are in Appendix \ref{sec:listings}.

\begin{figure}[b]
\begin{center}
\includegraphics[width=0.75\columnwidth]{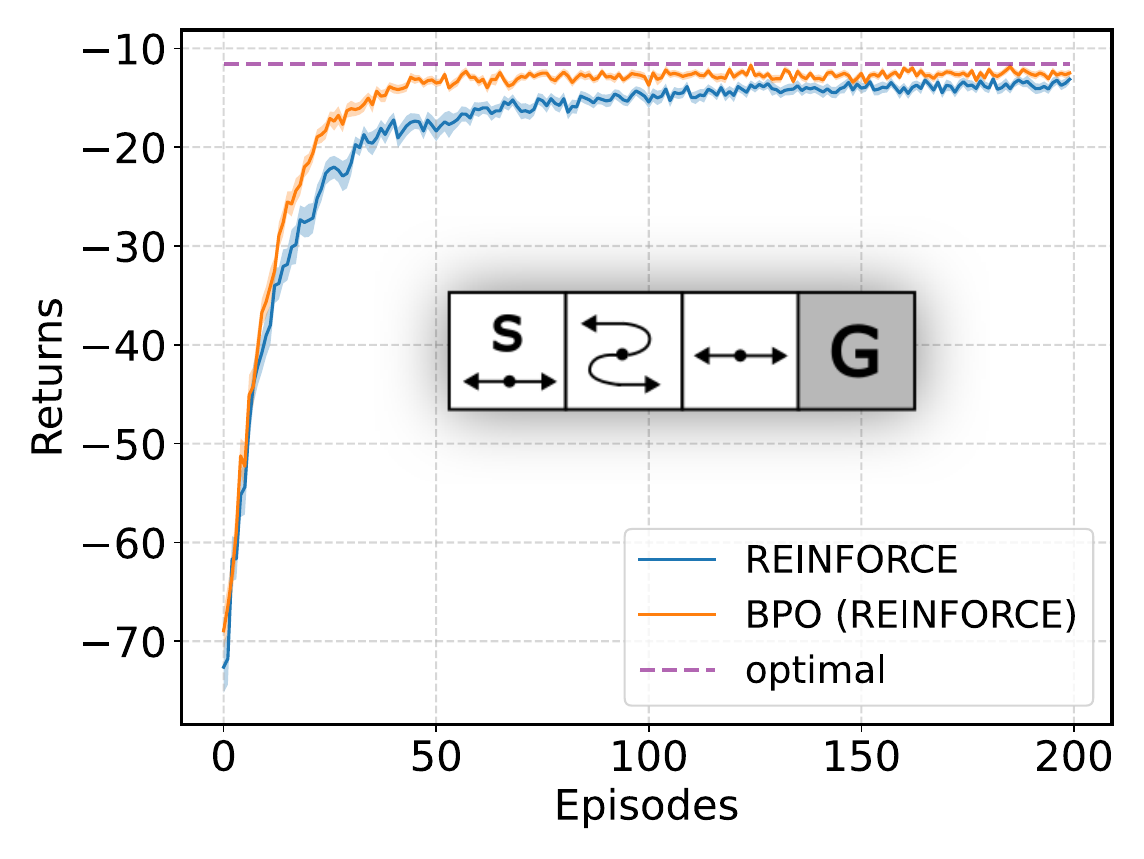}
\caption{\texttt{ShortCorridor} environment (center) and mean returns (10 eval episodes); averaged over 100 independent runs with standard error (SE) bars reported.}
\label{fig:reinforce}
\end{center}
\end{figure}

\begin{figure*}[t]
\begin{center}
\includegraphics[width=0.88\linewidth]{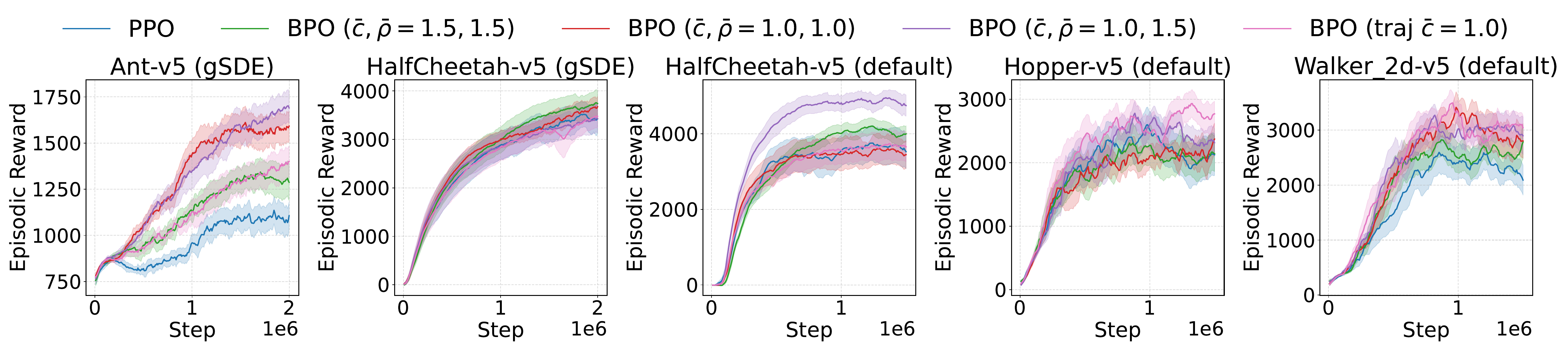}
\caption{MuJoCo mean returns (10 eval episodes); averaged over 10 independent runs with SE bars reported.}
\label{fig:ppo}
\end{center}
\end{figure*}

\begin{table*}[t]
\begin{center}
\begin{tabular}{lccccc}
\toprule
& Ant-v5 & \multicolumn{2}{c}{HalfCheetah-v5} & Hopper-v5 & Walker2d-v5 \\
& \texttt{(gSDE)} & \texttt{(gSDE)} & \texttt{(default)}& \texttt{(default)} & \texttt{(default)} \\
\midrule
PPO & $ 1106 \pm 111 $ & $ 3425 \pm 468 $ & $ 3527 \pm 670 $ &  $ 2126 \pm 492 $ & $ 2091 \pm 408 $\\
BPO ($\bar c, \bar \rho = 1.5,1.5$) & $\bm{1287 \pm 148}$ & $\mathit{3742 \pm 408}$ &  $\bm{3999 \pm 320}$&  $\mathit{2143 \pm 422}$  & $\bm{2782 \pm 456}$ \\
BPO ($\bar c, \bar \rho = 1.0,1.0$) &  $\bm{1592 \pm 126}$ & $\mathit{3674 \pm 321}$ &   $ 3449 \pm 550 $ &  $\mathit{2373 \pm 377}$ & $\bm{2805 \pm 470}$ \\
BPO ($\bar c, \bar \rho = 1.0,1.5$) & $\bm{1690 \pm 125}$ & $\mathit{3431 \pm 431}$ & $\bm{4749 \pm 419}$ & $\mathit{2316 \pm 326}$ & $\bm{2898 \pm 513}$ \\
BPO (\texttt{traj} $\bar c = 1.0$) & $\bm{1400 \pm 122}$ & $\mathit{3477 \pm 356}$&  $\mathit{3656 \pm 588}$ & $\bm{2770 \pm 296}$ & $\bm{3044 \pm 332}$ \\
\bottomrule 
\end{tabular}
\caption{MuJoCo mean returns (10 eval episodes after training) with two hyperparameter settings: (\texttt{default}) and (\texttt{gSDE}); \textbf{bold} denotes a statistically significant improvement over the baseline; \textit{italic} denotes non-statistically significant improvements.} 
\label{tab:mujocoresults}
\end{center}
\end{table*}

\paragraph{REINFORCE.} The action space for \texttt{ShortCorridor} grid world (c.f., Figure \ref{fig:reinforce}) is $\mathcal{A} =\{ \text{left}, \text{right}\}$. All states, $\mathcal{S} =\{0,1,2,3\}$ are identical under function approximation, $\mathbf{x}(s, \text{left}) = [0,1]^T$ and $\mathbf{x}(s, \text{right}) = [1,0]^T$. The problem is challenging as in state $1$ the left and right actions are reversed. Given that the agent can't distinguish states, the optimal policy is a stochastic policy picking right with probability $0.59$, achieving an optimal expected return of $-11.6$. 
We compare BPO with REINFORCE to standard REINFORCE \cite{williams1992simple}. Both the learning rate and underlying update rule for $\pi_{\theta}$ are fixed for a fair comparison. We see consistent policy improvement for both algorithms in the early stages of training, in the later stages of training BPO is more stable and converges faster, indicating the behaviour policy $\mu$ has converged to the optimal sampling distribution. The best truncation parameters we found here were $\bar c=1.0$ and $\bar \rho = 1.5$. Even though the advantage here is small, BPO appears more stable towards the end of training, thus we could possibly push the learning rate further to its limit. 

\paragraph{PPO.} We now evaluate BPO on several MuJoCo \cite{todorov2012mujoco} environments provided by Gymnasium \cite{towers2025gymnasium}, testing also, different hyperparameter settings for the underlying PPO algorithm. In particular we test with the standard hyperparameter settings (\texttt{default}) provided in \cite{schulman2017proximalpolicyoptimizationalgorithms}, we also test with generalized state-dependent exploration (\texttt{gSDE}) \cite{raffin2022smooth}. All hyperparameter settings are detailed in Appendix \ref{sec:additionalimplementation}, we note that for a fair comparison between BPO and PPO, we keep the underlying PPO parameters identical and we only tune the BPO specific hyperparameters (e.g., $\bar c$, $\bar \rho$), although in practice one might be able to push the PPO hyperparameters further (e.g., $\lambda \to 1$) as BPO reduces variance. The results are summarized in Table \ref{tab:mujocoresults}. We also provide complete learning curves for each environment, see Figure \ref{fig:ppo}.

\paragraph{Discussion.} In our experiments we search over various hyperparameters settings for the truncation parameters $\bar c$ and $\bar \rho$, we also experiment with trajectory level truncation (\texttt{traj}), where the full product $(\prod^{k-1}_{i=t}c_i)\rho_k$ in (\ref{eq:tislamda}) is truncated with $\bar c$, rather than truncating each individual IS ratio. The advantage of this, is ratios greater than $\bar c$ can propagate through the trajectory, but the product is always truncated for stable learning. Again, this can be efficiently computing by stepping backward through time. In Table \ref{tab:mujocoresults} we see that almost all configurations of BPO improve upon the baseline (PPO), and in most cases these improvements are statistically significant. Studying Figure \ref{fig:ppo} we see that BPO often exhibits faster convergence at the start of training and is more stable towards the end, demonstrating the well-designed behaviour policies can not only theoretically but empirically improve the convergence and stability of online RL, by reducing the variance of the returns estimator. We conduct some ablations for BPO, to observe the effect of removing the $\hat q_{\pi}$ estimate from the loss function in (\ref{eq:mucontloss}) and investigating the effect of \textit{symlog targets}. For these results please see Appendix \ref{sec:ablationstudies}.

\section{Conclusion}

In this paper we demonstrate that well designed behaviour policies can provably reduce the variance of practical returns estimators for online RL. We develop a sound methodology for training variance-reducing behaviour policies which are used to collect off-policy data for concurrently optimizing a target policy. Empirically, we show that our approach enables stable and more efficient off-policy RL, compared to the on-policy counterparts. Directions for future work include, extending our methodology to other policy-gradient algorithms, e.g., A2C \cite{mnih2016asynchronous}. It would also be worth exploring if our methodology is useful for value-based methods, e.g., DDPG \cite{silver2014deterministic}, TD3 \cite{fujimoto2018addressing} and SAC \cite{haarnoja2018soft}. We might also use model-based RL \cite{sutton1990integrated} to assist in designing globally optimal behaviour policies.

\section*{Acknowledgments} The research described in this paper was partially supported by the EPSRC (grant number EP/X015823/1).

\bibliography{aaai2026}

\newpage
\onecolumn
\appendix

\section{Technical Proofs}

\subsection{Preliminary Lemmas}

We start this section by providing proofs for the preliminary Lemmas: Lemma \ref{lem:optimal} and Lemma \ref{lem:zero} from (\citeauthor{owen2013monte} \citeyear{owen2013monte}, Ch.~9.1).

\begin{restatedlem}{\ref{lem:optimal} (restated) (\citeauthor{owen2013monte} \citeyear{owen2013monte}, Ch.~9.1)} $\forall x \in \mathcal{X}$ let $q^*(x) \propto p(x)\vert f(x)\vert$. Then $q^*$ is an optimal solution to (\ref{eq:problem}). 
\end{restatedlem}
\begin{proof}
First, for a given $p$ and $f$, define,
\begin{equation*}
    \mathcal{X}_{+} = \{ x \mid p(x)f(x) =0 \} 
\end{equation*}
For any sampling distribution $q \in \Lambda$, the variance of $I_{\text{IS}} = \rho(x)f(x)$, is given by,
\begin{align*}
    &\mathbb{V}_{X \sim q}(\rho(X)f(X))\\
    & = \mathbb{E}_{X \sim q}[(\rho(X)f(X))^2] - \left( \mathbb{E}_{X \sim q}[\rho(X)f(X)] \right)\\
    & = \mathbb{E}_{X \sim q}[(\rho(X)f(X))^2] - \left( \mathbb{E}_{X \sim p}[f(X)] \right) \tag*{(change of measure)}\\
    & = \int_{x: q(x) > 0, x\in \mathcal{X}_+} \frac{p^2(x)f^2(x)}{q(x)} - \left( \mathbb{E}_{X \sim p}[f(X)] \right) \tag*{(\text{$p(x)f(x)=0$ for all $x \not \in \mathcal{X}_{+}$})}\\
    & = \int_{x\in \mathcal{X}_+} \frac{p^2(x)f^2(x)}{q(x)} - \left( \mathbb{E}_{X \sim p}[f(X)] \right)\\
\end{align*}
The second term is now unrelated to the sampling distribution $q$, thus the optimization problem in (\ref{eq:problem}) can be equivalently written as,
\begin{equation*}
    \min_{q \in \Lambda} \int_{x\in \mathcal{X}_+} \frac{p^2(x)f^2(x)}{q(x)}
\end{equation*}
In the case that $|\mathcal{X}_+|=0$, then the variance is always zero and so any $q$ is optimal. Thus, we continue for $|\mathcal{X}_+|>0$, we note that $\Lambda$ here can be equivalently expressed as,
\begin{equation*}
    \Lambda = \{ q \in \Delta(\mathcal{X}) \mid \forall x \in \mathcal{X}_{+}, q(a) > 0 \}
\end{equation*}
Thus the optimization problem can be written again equivalently as,
\begin{align}
\begin{split}
    &\min_{q \in \Delta(\mathcal{X})} \int_{x\in \mathcal{X}_+} \frac{p^2(x)f^2(x)}{q(x)}\\
    & \text{subject to}\quad p(x) > 0\quad \forall x \in \mathcal{X}_+
\end{split} \label{eq:equivproblem}
\end{align}
It is not to hard to see that any optimal solution for (\ref{eq:equivproblem}) must put all its probability density on $\mathcal{X}_+$, since if $\int_{x\in \mathcal{X}_+} q(x) < 1$, then there must be some $x \not \in \mathcal{X}_{+}$ with density $p(x) > 0$, since $x$ does not contribute to the integral in the objective function, one can always decrease the objective by putting more density on some $x' \in \mathcal{X}_+$. Thus, we write the following problem,
\begin{align}
    \begin{split}
    &\min_{z \in \Delta(\mathcal{X}_+)} \int_{x\in \mathcal{X}_+} \frac{p^2(x)f^2(x)}{z(x)}\\
    & \text{subject to}\quad z(x) > 0 \quad \forall x \in \mathcal{X}_+
\end{split} \label{eq:equivproblem2}
\end{align}
If $z^*$ is an optimal solution to (\ref{eq:equivproblem2}) then an optimal solution to (\ref{eq:equivproblem}) can be constructed as follow,s
\begin{equation}
    q^*(x) = \begin{cases}
        z^*(x) & \text{if $x \in \mathcal{X}_+$}\\
        0 & \text{otherwise}
    \end{cases} \label{eq:cases}
\end{equation}
For any $z$, such that $\forall x \in \mathcal{X}_+,z(x) > 0 $, the Cauchy-Schwarz inequality shows that,
\begin{equation*}
    \left(\int_{x \in \mathcal{X}_+} \frac{p^2(x)f^2(x)}{z(x)} \right) \left( \int_{x \in \mathcal{X}_+} z(x) \right) \geq \left( \int_{x \in \mathcal{X}_+} \frac{p(x)\vert f(x)\vert}{\sqrt{z(x)}}\textstyle{\sqrt{z(x)}} \right)^2 = \left( \int_{x \in \mathcal{X}_+} p(x)\vert f(x)\vert\right)^2
\end{equation*}
with equality holding for,
\begin{equation*}
    z^*(x) = \frac{p(x)\vert f(x) \vert}{\int_{x' \in \mathcal{X}}p(x)\vert f(x)\vert} > 0 \tag*{(since \text{$\vert \mathcal{X}_+ \vert > 0$})}
\end{equation*}
Since $\int_{x\in \mathcal{X}_+} z^*(x) =1$, we conclude $z^*$ is an optimal solution for (\ref{eq:equivproblem2}) and thus an optimal solution $q^*$ for (\ref{eq:equivproblem}) can be constructed by (\ref{eq:cases}). Noting that $p(x)\vert f(x) \vert=0$ for $x \not \in \mathcal{X}_+$, then $q^*$ can be expressed as,
\begin{equation*}
    q^* = \frac{p(x)\vert f(x) \vert}{\int_{x' \in \mathcal{X}}p(x')\vert f(x') \vert}
\end{equation*}
\end{proof}

\begin{restatedlem}{\ref{lem:zero} (restated) (\citeauthor{owen2013monte} \citeyear{owen2013monte}, Ch.~9.1)}
If $\forall x \in \mathcal{X}, f(x) \geq 0$ or $\forall x \in \mathcal{X}, f(x) \leq 0$, then $\Lambda = \Lambda_{+}$ (where $\Lambda_{+}$ denotes the set of all distributions giving unbiased estimators), and $q^*$ (defined in Lemma \ref{lem:optimal}) gives zero variance, i.e., $\mathbb{V}_{q^*}(I_{\text{IS}}) =0$. 
\end{restatedlem}
\begin{proof} We start by providing the formal definition for $\Lambda_+$,
\begin{equation}
    \Lambda_+ = \{ q \in \Delta(\mathcal{X}) \mid \mathbb{E}_{X \sim q}[\rho(X) f(X)] = \mathbb{E}_{X \sim p}[f(X)]\}
\end{equation}
which corresponds to all sampling distributions $q$ that give unbiased estimators for $\mathbb{E}_{X \sim p}[f(X)]$. The $q \in \Lambda \Rightarrow q \in \Lambda_+$ direction is straightforward,
\begin{align*}
    \mathbb{E}_{X \sim q}[\rho(X) f(X)] &= \int_{x : q(x)>0} q(x)\frac{p(x)}{q(x)}f(x)\\
    &= \int_{x : q(x)>0}p(x)f(x)\\
    &= \int_{x : q(x)>0}p(x)f(x) + \int_{x : q(x)=0}p(x)f(x) \tag*{(since \text{$q \in \Lambda$})}\\
    &= \int_{x \in \mathcal{X}}p(x)f(x)\\
    &= \mathbb{E}_{X \sim p}[f(X)]\\
\end{align*}
Now we show the $q \in \Lambda_+ \Rightarrow q \in \Lambda$ direction. For any $q \in \Lambda_+$, we have,
\begin{equation*}
    \int_{x : q(x)>0} q(x)\frac{p(x)}{q(x)}f(x) = \int_{x \in \mathcal{X}}p(x)f(x)
\end{equation*}
which implies that,
\begin{equation*}
    \int_{x : q(x)=0} p(x)f(x) = 0
\end{equation*}
Since $p(x) \geq 0$ (probability denisty) and $f(x)$ has the same sign (by assumption in the lemma statement), then we must have,
\begin{equation*}
    p(x)f(x) = 0 \quad \forall x \in \{x \mid q(x) =0 \}
\end{equation*}
This gives exactly $q(x) = 0 \Rightarrow p(x)f(x) =0$ and thus $q \in \Lambda$. In remains now to show zero variance. First using the same notation from the proof of Lemma \ref{lem:optimal}, we define,
\begin{equation*}
    \mathcal{X}_{+} = \{ x \mid p(x)f(x) =0 \} 
\end{equation*}
Then when $\forall x \in \mathcal{X}, f(x) \geq 0$ and when $\vert \mathcal{X}_+ \vert > 0$ then we have,
\begin{equation*}
    q^*(x) = \frac{p(x) \vert f(x) \vert}{c} \quad \forall x \in \mathcal{X}
\end{equation*}
where $c = \int_{x' \in \mathcal{X}} = p(x')\vert f(x') \vert > 0 $ is the normalizing constant. Plugging $q^*$ into the IS estimator, $I_{\text{IS}}$ gives us,
\begin{equation*}
  I_{\text{IS}} = \rho(x)f(x) = \frac{p(x)}{q^*(x)} f(x) = \frac{p(x)}{(p(x) \vert f(x) \vert)/c}f(x) = c
\end{equation*}
This means when $\forall x \in \mathcal{X}, f(x) \geq 0$, with optimal sampling distribution $q^*$, the IS estimator $I_{\text{IS}}$ is a constant function, thus,
\begin{equation*}
    \mathbb{V}_{q^*}(I_{\text{IS}}) = 0
\end{equation*}
Finally, in the case when $\vert \mathcal{X}_+ \vert = 0$, as argued in Lemma \ref{lem:optimal}, the variance is always zero, this can be easily seen by,
\begin{equation*}
     I_{\text{IS}} = \rho(x)f(x) = \frac{p(x)}{q(x)} f(x) = 0 \tag*{(for any \text{$q \in \Lambda$})}
\end{equation*}
For $\forall x \in \mathcal{X}, f(x) \leq 0$ this proof is similar and thus omitted here. 
\end{proof}

\subsection{Preface}

We preface this section of the appendix, by providing the following summary of definitions, starting from the TD($\lambda$)-style recursion presented in Remark \ref{rem:recursion}, we have,
\begin{equation}
    G_t^{\text{TIS}, \lambda} = v_{\pi}(S_t)  + \delta_t + \gamma\lambda c_t\left( G^{\text{TIS, $\lambda$}}_{t+1} - v_{\pi}(S_{t+1})\right) \label{eq:gtisrecursive}
\end{equation}
where,
\begin{equation*}
    \delta_t = \rho_t(r(S_t, A_t) + \gamma v_\pi(S_{t+1}) - v_\pi(S_{t}))
\end{equation*}
and the truncated IS ratios,
\begin{equation*}
    c_t := \min\left(\bar c, \frac{\pi(A_t \vert S_t)}{\mu(A_t \vert S_t)}\right) \qquad \rho_t := \min\left(\bar \rho, \frac{\pi(A_t \vert S_t)}{\mu(A_t \vert S_t)}\right)
\end{equation*}
For the target policy $\pi$, define,
\begin{equation*}
    J_\pi(s) := \mathbb{V}_{\pi}( G_t^{\text{TIS}, \lambda}\mid S_t =s)
\end{equation*}
Similarly, for the behaviour policy $\mu$, define,
\begin{equation*}
    J_\mu(s) := \mathbb{V}_{\mu}( G_t^{\text{TIS}, \lambda}\mid S_t =s)
\end{equation*}
For later convenience we define the variance Bellman operator,
\begin{align*}
    (\mathcal{J}_\mu f)(s)
    &= \mathbb{E}_{A_t \sim \mu}[\gamma^2c^2_t \mathbb{E}_{S_{t+1}\sim p}[f(S_{t+1}) \mid S_t, A_t] \mid S_t=s]\\
    &+ \mathbb{E}_{A_t \sim \mu}[\rho^2_t\nu_\pi(S_t, A_t) \mid S_t=s] + \mathbb{E}_{A_t \sim \mu}[\rho^2_tq^2_\pi(S_t, A_t) \mid S_t=s] - v^2_\pi(s)
\end{align*}
more concisely, let,
\begin{equation*}
    \mathcal{J}_\mu = \mathcal{K}_\mu\mathcal{J}_\mu + b_\mu
\end{equation*}
where,
\begin{align*}
    (\mathcal{K}_\mu f)(s) &= \mathbb{E}_{A_t \sim \mu}[\gamma^2 c_t^2 \mathbb{E}_{S_{t+1}\sim p}[f(S_{t+1}) \mid S_t, A_t] \mid S_t=s]\\
     b_\mu(s) &= \mathbb{E}_{A_t \sim \mu}[\rho_t^2 \nu_\pi(S_t, A_t) \mid S_t=s] + \mathbb{E}_{A_t \sim \mu}[\rho^2_tq^2_\pi(S_t, A_t) \mid S_t=s] - v^2_\pi(s)\\
\end{align*}
For convenience we also provide the definition of the next step variance,
\begin{equation}
    \nu_\pi(s, a) = \mathbb{V}_{S_{t+1}\sim p }(\gamma v_\pi(S_{t+1}) \mid S_t = s, A_t = a) \label{eq:defnu}
\end{equation}
\noindent Finally we define,
\begin{align}
    \tilde r_{\pi}(s, a) & := \nu_{\pi}(s, a) + q_{\pi}^2(s, a) - v^2_{\pi}(s) \label{eq:rtildedef}\\
    \tilde q_{\pi}(s, a) &= \tilde r_{\pi}(s, a) + \gamma^2 \mathbb{E}_{a' \sim \pi, s' \sim p}[\tilde q_{\pi}(s', a')] \label{eq:qtildedef}
\end{align}
\subsection{Auxiliary Lemmas}

We continue with some auxiliary lemmas.

\begin{lemma}(Variance Bellman operator) \label{lem:variancebellman} For any $\mu \in \Lambda$ (unbiased policies), $\bar c, \bar \rho = \infty$ and $\lambda=1$,
\begin{align*}
    & \mathbb{V}_{\mu}( G_t^{\text{TIS}, \lambda}\mid S_t =s)\\
    &= \mathbb{E}_{A_t \sim \mu}[\gamma^2c^2_t \mathbb{E}_{S_{t+1}\sim p}[\mathbb{V}_{\mu}(G_{t+1}^{\text{TIS}, \lambda} \mid S_{t+1}) \mid S_t, A_t] \mid S_t=s]\\
    &+ \mathbb{E}_{A_t \sim \mu}[\rho^2_t\nu_\pi(S_t, A_t) \mid S_t=s] + \mathbb{E}_{A_t \sim \mu}[\rho^2_tq^2_\pi(S_t, A_t) \mid S_t=s] - v^2_\pi(s)\\
    &= (\mathcal{J}_\mu \mathbb{V}_{\mu}( G_t^{\text{TIS}, \lambda}))(s)
\end{align*}    
\end{lemma}
\begin{proof}
    The proof is successive applications of the law of total variance, first conditioning on $A_t \sim \mu$,
    \begin{equation*}
        \mathbb{V}_{\mu}(G_t^{\text{TIS}, \lambda}\mid S_t = s) = \underbrace{\mathbb{E}_{A_t\sim \mu}[\mathbb{V}_\mu(G_t^{\text{TIS}, \lambda} \mid S_t, A_t)\mid S_t] }_\text{conditional variance} + \underbrace{\mathbb{V}_{A_t\sim\mu}(\mathbb{E}_{\mu}[G_t^{\text{TIS}, \lambda}\mid S_t, A_t] \mid S_t)}_\text{variance between actions}
    \end{equation*}
Dealing with the \emph{conditional variance} term, 
\begin{align*}
    &\mathbb{V}_\mu(G_t^{\text{TIS}, \lambda} \mid S_t, A_t)\\
    &= \mathbb{V}_\mu( v_\pi(S_t) + \delta_t + \gamma\lambda c_t( G^{\text{TIS, $\lambda$}}_{t+1} - v_{\pi}(S_{t+1})) \mid S_t, A_t)\tag*{(by recursive definition (\ref{eq:gtisrecursive}))}\\
    &= \mathbb{V}_\mu(\gamma\rho_tv_\pi(S_{t+1}) + \gamma\lambda c_t( G^{\text{TIS, $\lambda$}}_{t+1} - v_{\pi}(S_{t+1})) \mid S_t, A_t) \tag*{(dropping deterministic terms)}\\
    &= \mathbb{V}_\mu(\gamma((\rho_t - \lambda c_t) v_\pi(S_{t+1}) + \lambda c_t( G^{\text{TIS, $\lambda$}}_{t+1})) \mid S_t, A_t) \tag*{(collecting terms)}
\end{align*}
Now applying the law of total variance over the next step transitions $S_{t+1} \sim p$,
\begin{align*}
    &\mathbb{V}_\mu(\gamma((\rho_t - \lambda c_t) v_\pi(S_{t+1}) + \lambda c_t( G^{\text{TIS, $\lambda$}}_{t+1})) \mid S_t, A_t)\\
    \begin{split}
    &= \mathbb{E}_{S_{t+1} \sim p}[\mathbb{V}_\mu(\gamma((\rho_t - \lambda c_t) v_\pi(S_{t+1}) + \lambda c_t( G^{\text{TIS, $\lambda$}}_{t+1})) \mid S_{t+1}) \mid S_t, A_t ] \\
    &\quad + \mathbb{V}_{S_{t+1}\sim \mu}(\mathbb{E}_\mu[\gamma((\rho_t - \lambda c_t) v_\pi(S_{t+1}) + \lambda c_t( G^{\text{TIS, $\lambda$}}_{t+1})) \mid S_{t+1}] \mid S_t, A_t)
    \end{split}\\
    & = \mathbb{E}_{S_{t+1} \sim p}[\gamma^2 c_t^2\mathbb{V}_\mu( G^{\text{TIS, $\lambda$}}_{t+1} \mid S_{t+1}) \mid S_t, A_t ] + \mathbb{V}_{S_{t+1}\sim \mu}(\gamma((\rho_t - \lambda c_t) v_\pi(S_{t+1}) + \mathbb{E}_\mu[\lambda c_t( G^{\text{TIS, $\lambda$}}_{t+1})) \mid S_{t+1}] \mid S_t, A_t) \tag*{(\text{$\lambda=1$})}\\
    & = \mathbb{E}_{S_{t+1} \sim p}[\gamma^2 c_t^2\mathbb{V}_\mu( G^{\text{TIS, $\lambda$}}_{t+1} \mid S_{t+1}) \mid S_t, A_t ] + \mathbb{V}_{S_{t+1}\sim \mu}(\gamma((\rho_t - \lambda c_t) v_\pi(S_{t+1}) + \lambda c_t v_\pi(S_{t+1})) \mid S_t, A_t) \tag*{(unbiasedness)}\\
    & = \mathbb{E}_{S_{t+1} \sim p}[\gamma^2 c_t^2\mathbb{V}_\mu( G^{\text{TIS, $\lambda$}}_{t+1} \mid S_{t+1}) \mid S_t, A_t ] + \mathbb{V}_{S_{t+1}\sim \mu}(\gamma\rho_t v_\pi(S_{t+1}) \mid S_t, A_t)\\
    & = \mathbb{E}_{S_{t+1} \sim p}[\gamma^2 c_t^2\mathbb{V}_\mu( G^{\text{TIS, $\lambda$}}_{t+1} \mid S_{t+1}) \mid S_t, A_t ] + \rho^2_t\nu_\pi(s, a)\\
\end{align*}
Therefore, we have,
\begin{equation*}
    \mathbb{E}_{A_t\sim \mu}[\mathbb{V}_\mu(G_t^{\text{TIS}, \lambda}\mid S_t, A_t)\mid S_t] = \mathbb{E}_{A_t\sim \mu}[\gamma^2 c_t^2\mathbb{E}_{S_{t+1} \sim p}[\mathbb{V}_\mu( G^{\text{TIS, $\lambda$}}_{t+1} \mid S_{t+1}) \mid S_t, A_t ] \mid S_t] + \mathbb{E}_{A_t\sim \mu}[\rho_t^2\nu_\pi(S_t, A_t) \mid S_t]
\end{equation*}
Now dealing with the \emph{variance between actions} terms, firstly,
\begin{align*}
    &\mathbb{E}_{\mu}[G_t^{\text{TIS}, \lambda}\mid S_t, A_t] \\
    & = \mathbb{E}_{\mu}[v_\pi(S_t) +\delta_t + \gamma\lambda c_t( G^{\text{TIS, $\lambda$}}_{t+1} - v_{\pi}(S_{t+1}))\mid S_t, A_t]\\
    & = \mathbb{E}_{\mu}[v_\pi(S_t) +\delta_t\mid S_t, A_t] \tag*{(unbiasedness)}\\
    & = v_\pi(s) + \mathbb{E}_{\mu}[\rho_t (r(S_t, A_t) + \gamma v_\pi(S_{t+1}) - v_\pi(S_t))\mid S_t, A_t]\\
    & = v_\pi(s) - v_\pi(s)\mathbb{E}[\rho_t] - \mathbb{E}_{\mu}[\rho_t (r(S_t, A_t) + \gamma v_\pi(S_{t+1}))\mid S_t, A_t]\\
    & = \mathbb{E}_{\mu}[\rho_t (r(S_t, A_t) + \gamma v_\pi(S_{t+1}))\mid S_t, A_t]\ \tag*{(since \text{$\mathbb{E}_{\mu}[\rho_t]=1$})}\\
    & = \mathbb{E}_{\mu}[\rho_t q_\pi(S_t, A_t)\mid S_t, A_t]\\
\end{align*}
Therefore, we have,
\begin{align*}
    &\mathbb{V}_{A_t \sim \mu}(\mathbb{E}_{\mu}[G_t^{\text{TIS}, \lambda}\mid S_t, A_t] \mid S_t)\\
    &= \mathbb{V}_{A_t \sim \mu}(\mathbb{E}_{\mu}[\rho_t q_\pi(S_t, A_t))\mid S_t, A_t] \mid S_t )\\
    & = \mathbb{E}_{A_t \sim \mu}[\rho_t^2 q^2_\pi(S_t, A_t) \mid S_t] - (\mathbb{E}_{A_t \sim \mu}[\rho_t q_\pi(S_t, A_t)\mid S_t])^2\\
    & = \mathbb{E}_{A_t \sim \mu}[\rho_t^2 q^2_\pi(S_t, A_t)\mid S_t] - v_\pi^2(s)
\end{align*}
Putting both terms together now gives the desired result. 
\end{proof}

\begin{lemma}[Variance equality for $\pi$] \label{lem:variancequality} For $\bar c, \bar\rho \geq 1$
\begin{equation*}
    \mathbb{V}_{\pi}(G^{\text{TIS}, \lambda}_t \mid S_t = s) = \mathbb{E}_{A_t \sim \pi}[\tilde q (S_t, A_t)]
\end{equation*} 
\end{lemma}
\begin{proof} Let $J_\pi =  \mathbb{V}_{\pi}(G^{\text{TIS}, \lambda}_t \mid S_t = s)$, recalling, the definition for the variance Bellman operator $\mathcal{J}_\mu$,
\begin{align*}
    (\mathcal{J}_\mu f)(s)
    &= \mathbb{E}_{A_t \sim \mu}[\gamma^2c^2_t \mathbb{E}_{S_{t+1}\sim p}[f(S_{t+1}) \mid S_t, A_t] \mid S_t=s]\\
    &+ \mathbb{E}_{A_t \sim \mu}[\rho^2_t\nu_\pi(S_t, A_t) \mid S_t=s] + \mathbb{E}_{A_t \sim \mu}[\rho^2_tq^2_\pi(S_t, A_t) \mid S_t=s] - v^2_\pi(s)
\end{align*}
and,
\begin{align*}
    \tilde r_{\pi}(s, a) & :=  \nu_{\pi}(s, a) + q_{\pi}^2(s, a) - v^2_{\pi}(s, a)\\
    \tilde q_{\pi}(s, a) &= \tilde r_{\pi}(s, a) + \gamma^2 \mathbb{E}_{a' \sim \pi, s' \sim p}[\tilde q_{\pi}(s', a')]
\end{align*}
Since $\mu = \pi$, and $\bar c, \bar\rho \geq 1$ then,
\begin{align*}
    (\mathcal{J}_\pi J_\pi)(s)
    &= \mathbb{E}_{A_t \sim \mu}[\gamma^2 \mathbb{E}_{S_{t+1}\sim p}[J_\pi(S_{t+1}) \mid S_t, A_t] \mid S_t=s]\\
    &+ \mathbb{E}_{A_t \sim \mu}[\nu_\pi(S_t, A_t) \mid S_t=s] + \mathbb{E}_{A_t \sim \mu}[q^2_\pi(S_t, A_t) \mid S_t=s] - v^2_\pi(s)
\end{align*}
Lemma \ref{lem:variancebellman} says that $J_\pi = \mathcal{J}_\pi J_\pi$. Since $\Vert \mathcal{J}_\pi f - \mathcal{J}_\pi f' \Vert_{\infty} \leq \gamma^2 \Vert f - f' \Vert_{\infty}$, as $\mathcal{J}_\pi$ is a $\gamma^2$-contraction on bounded functions, thus the fixed point $J_\pi$ is unique. Now let $\tilde v_\pi(s) := \mathbb{E}_{a\sim \pi}[\tilde q (s, a)]$, then,
\begin{align*}
    &\tilde v_\pi(s)\\
    &= \mathbb{E}_{a \sim \pi}[\tilde q (s, a)]\\
    &= \mathbb{E}_{a \sim \pi}[\tilde r_\pi(s, a)  + \gamma^2 \mathbb{E}_{a'\sim\pi, s'\sim p}[\tilde q (s, a)]] \tag*{(by (\ref{eq:qtildedef}))}\\
    &= \mathbb{E}_{a \sim \pi}[\nu_\pi(s, a) + q^2(s, a) - v^2(s, a)  + \gamma^2 \mathbb{E}_{a'\sim\pi, s'\sim p}[\tilde q (s, a)]] \tag*{(by (\ref{eq:rtildedef}))}\\
    & = \mathbb{E}_{a \sim \pi}[\gamma^2\mathbb{E}_{a'\sim\pi, s'\sim p}[\tilde q (s, a)]] + \mathbb{E}_{a \sim \pi}[\nu_\pi(s, a)] + \mathbb{E}_{a \sim \pi}[q^2_\pi(s, a)] - \mathbb{E}_{a \sim \pi}[v_\pi^2(s)]\\
    & = \mathbb{E}_{a \sim \pi}[\gamma^2\mathbb{E}_{s' \sim p}[\tilde v_\pi(s')]] + \mathbb{E}_{a \sim \pi}[\nu_\pi(s, a)] + \mathbb{E}_{a \sim \pi}[q^2_\pi(s, a)] - v_\pi^2(s)\\
\end{align*}
which establishes that $\tilde v \mathcal{J}_\pi = \tilde v$. And thus by uniqueness $\tilde v = J_\pi$.
\end{proof}
\begin{lemma} \label{lem:qhattildedef}
    \begin{equation*}
    \begin{split}
        \hat q_{\pi}(s, a) &= \nu(s, a) + q^2_{\pi}(s,a ) +  \gamma^2\mathbb{E}_{s' \sim p}[\mathbb{E}_{a' \sim \pi}[\tilde q(s', a')]]\\
    & = \tilde q_{\pi}(s, a) + v^2_\pi(s)
    \end{split} 
\end{equation*} 
\end{lemma}
\begin{proof}
    The proof is a straightforward manipulation of (\ref{eq:qhatdef}).
\end{proof}

\noindent We continue now with the main proofs from the paper.
 
\subsection{Analysis for Lemma \ref{lem:vtrace}}
\label{sec:lemvtrace}

\noindent 
\noindent Let $\bar c, \bar \rho \geq 1$ and $\mu = \pi$, then from (\ref{eq:tislamda}), 
\begin{align*}
    &\mathbb{E}_{\mu}[G_t^{\text{TIS}, \lambda}] \\
    &=\mathbb{E}_{\pi}[G_t^{\text{TIS}, \lambda}] \tag*{(on policy)} \\
    & = \mathbb{E}_{\pi}\left[v_{\pi}(S_t) + \sum^{\infty}_{k=t} (\gamma\lambda)^{k-t} \left(\prod_{i=t}^{k-1} c_i \right) \delta_k \right]\\
    & = \mathbb{E}_{\pi}\left[v_{\pi}(S_t) + \sum^{\infty}_{k=t} (\gamma\lambda)^{k-t} (r(S_t, A_t) + \gamma v_\pi(S_{t+1}) - v_\pi(S_t))\right] \tag*{(ratios are 1)}\\
    & = \mathbb{E}_{\pi}[v_{\pi}(S_t)] + \sum^{\infty}_{k=t} (\gamma\lambda)^{k-t} \mathbb{E}_\pi[r(S_t, A_t) + \gamma v_\pi(S_{t+1}) - v_\pi(S_t)]  \\
    &= \mathbb{E}_{\pi}[v_{\pi}(S_t)] \tag*{(since \text{$v_\pi$} satisfies the Bellman equation)}\\
    & = \mathbb{E}_\pi(\mathbb{E}_\pi(G_t \mid S_t)) = \mathbb{E}_\pi[G_t] \tag*{(tower property)}
\end{align*}

\subsection{Proof of Theorem \ref{thm:unbiasedness}}
\label{sec:thmunbiasedness}
\begin{restatedthm}{\ref{thm:unbiasedness} (restated)}
    Let $\bar c, \bar \rho = \infty$; then for all $\mu \in \Lambda$, $s \in \mathcal{S}$, $\lambda \in [0,1]$, we have,
    \begin{equation*}
        \mathbb{E}_{\mu}[ G_t^{\text{TIS}, \lambda} \mid S_t = s ] = v_{\pi}(s)
    \end{equation*}
\end{restatedthm}
\begin{proof}

Let $\bar c, \bar \rho = \infty$ then from (\ref{eq:tislamda}),

\begin{align*}
    &\mathbb{E}_{\mu}[G_t^{\text{TIS}, \lambda} \mid S_t=s] \\
    & = \mathbb{E}_{\mu}\left[v_{\pi}(S_t) + \sum^{\infty}_{k=t} (\gamma\lambda)^{k-t} \left(\prod_{i=t}^{k-1} c_i \right) \delta_k \mid S_t =s\right]\\
    & = \mathbb{E}_{\mu}\left[v_{\pi}(S_t) + \sum^{\infty}_{k=t} (\gamma\lambda)^{k-t}  \rho_{t:k-1} \delta_k \mid S_t=s \right]\\
    & = \mathbb{E}_{\mu}\left[\rho_t(R_{t+1} + \gamma v_\pi(S_{t+1})) + v_{\pi}(S_t) - \rho_t v_\pi(S_t) + \sum^{\infty}_{k=t+1} (\gamma\lambda)^{k-t}  \rho_{t:k-1} \delta_k \mid S_t=s\right]\\
    & = \mathbb{E}_{\mu}\left[\rho_tq_\pi(S_t, A_t)) + v_{\pi}(S_t) - \rho_t v_\pi(S_t) + \sum^{\infty}_{k=t+1} (\gamma\lambda)^{k-t}  \rho_{t:k-1} \delta_k \mid S_t=s\right] \\
    & = \mathbb{E}_{\mu}[\rho_tq_\pi(S_t, A_t)) \mid S_t=s] + v_\pi(s) - \mathbb{E}_\mu[\rho_t v_\pi(S_t) \mid S_t=s] + \mathbb{E}_\mu\left[\sum^{\infty}_{k=t+1} (\gamma\lambda)^{k-t}  \rho_{t:k-1} \delta_k \mid S_t=s\right] \\
    & = \mathbb{E}_{\mu}[\rho_tq_\pi(S_t, A_t)) \mid S_t=s] + v_\pi(s) -  v_\pi(S_t)\mathbb{E}_\mu[\rho_t] + \mathbb{E}_\mu\left[\sum^{\infty}_{k=t+1} (\gamma\lambda)^{k-t}  \rho_{t:k-1} \delta_k \mid S_t=s\right] \\
    & = \mathbb{E}_{\mu}[\rho_tq_\pi(S_t, A_t)) \mid S_t=s] + \mathbb{E}_\mu\left[\sum^{\infty}_{k=t+1} (\gamma\lambda)^{k-t}  \rho_{t:k-1} \delta_k \mid S_t=s\right] \tag*{(since \text{$\mathbb{E}_\mu[\rho_t] =1$})}\\
    & = \mathbb{E}_{\mu}[\rho_tq_\pi(S_t, A_t)) \mid S_t=s] + \mathbb{E}_\mu\left[\sum^{\infty}_{k=t+1} (\gamma\lambda)^{k-t}  \rho_{t:k} (R_{k+1} + \gamma v_\pi(S_{k+1}) - v_\pi(S_k)) \mid S_t=s\right]\\
    & = \mathbb{E}_{\mu}[\rho_tq_\pi(S_t, A_t)) \mid S_t=s] \tag*{(since \text{$v_\pi$} satisfies the Bellman equation)}\\
    &= \int_a \mu(a \mid s) \frac{\pi(a \mid s)}{\mu(a \mid s)}q_\pi(s, a)\\
     &=\mathbb{E}_\pi[q(S_t, A_t) \mid S_t =s] \tag*{(recalling \text{$\mu(a \mid s) = 0 \Rightarrow \pi(a \mid s)q(s, a) = 0$} from Lemma \ref{lem:optimal})}\\ 
     &= v_\pi(s)
\end{align*}
\end{proof}
\subsection{Proof of Theorem \ref{thm:variancereduction}}
\label{sec:thmvariancereduction}

\begin{restatedthm}{\ref{thm:variancereduction} (restated)}
    For any $s \in \mathcal{S}$, $\gamma \in [0,1)$, and $\lambda = 1$, $\bar c, \bar \rho = \infty$,
    \begin{equation*}
        \mathbb{V}_{\hat \mu}\left(G_t^{\text{TIS}, \lambda} \mid S_t =s \right) \leq  \mathbb{V}_{\pi}\left(G_t^{\text{TIS}, \lambda} \mid S_t = s\right) - \epsilon(s)
    \end{equation*}
Where,
\begin{equation*}
    c(s) := \textstyle{\mathbb{E}_{a \sim \pi}[\hat q_{\pi}(s, a)] - \left(\mathbb{E}_{a\sim \pi}[\sqrt{\hat q_{\pi}(s, a)}]\right)^2}
\end{equation*}
And,
\begin{equation*}
    \epsilon(s) := c(s) + \gamma^2\mathbb{E}_{A_t \sim \hat \mu_t} \left[ \rho_t^2 \mathbb{E}_{S_{t+1} \sim p}\left[\epsilon(S_{t+1})\right] \right]
\end{equation*}
\end{restatedthm}
\begin{proof}
Recall the variance Bellman operator,
\begin{equation*}
    \mathcal{J}_\mu = \mathcal{K}_\mu\mathcal{J}_\mu + b_\mu
\end{equation*}
where,
\begin{align*}
    (\mathcal{K}_\mu f)(s) &= \mathbb{E}_{A_t \sim \mu}[\gamma^2 c_t^2 \mathbb{E}_{S_{t+1}\sim p}[f(S_{t+1}) \mid S_t, A_t] \mid S_t=s]\\
     b_\mu(s) &= \mathbb{E}_{A_t \sim \mu}[\rho_t^2 \nu_\pi(S_t, A_t) \mid S_t=s] + \mathbb{E}_{A_t \sim \mu}[\rho^2_tq^2_\pi(S_t, A_t) \mid S_t=s] - v^2_\pi(s)\\
\end{align*}
In the on-policy case (i.e., $\mu=\pi$), and given $\bar c, \bar \rho \geq 1$, notice that $J_\pi =  \mathbb{V}_{\pi}(G^{\text{TIS}, \lambda}_t \mid S_t = s)$ is a fixed point for $\mathcal{J}_\pi$ (Lemma \ref{lem:variancequality}), observing that $\mathcal{K}_\pi$ is a $\gamma^2$-contraction, thus $J_\pi$ is the unique fixed point. Now denote the per state minimizer as,
\begin{equation*}
    c(s) = \textstyle{\mathbb{E}_{a \sim \pi}[\hat q_{\pi}(s, a)] - \left(\mathbb{E}_{a\sim \pi}[\sqrt{\hat q_{\pi}(s, a)}]\right)^2}
\end{equation*}
and the propagated term,
\begin{equation*}
    \epsilon(s) = c(s) + \gamma^2\mathbb{E}_{A_t \sim \hat \mu_t} \left[ \rho_t^2 \mathbb{E}_{S_{t+1} \sim p}\left[\epsilon(S_{t+1})\right] \right]
\end{equation*}
It remains to verify that $J_{\pi} - \epsilon$ is the unique fixed point for the off-policy variance Bellman operator $\mathcal{J}_{\hat \mu}$ with behaviour policy,
\begin{equation*}
    \hat \mu(a \vert s) \propto \pi(a \vert s) \textstyle{\sqrt{\hat q_{\pi}(s, a)}}
\end{equation*}
We continue by inspection,
\begin{align*}
    &\mathcal{J}_{\hat \mu} (J_\pi - \epsilon)(s) \\
    &= \mathcal{K}_{\hat \mu} (J_\pi - \epsilon)(s) + b_{\hat \mu}\\
    \begin{split}
        &= \mathbb{E}_{A_t \sim \hat \mu}[\gamma^2 c_t^2\mathbb{E}_{S_{t+1}\sim p}[J_\pi(S_{t+1}) \mid S_t, A_t] \mid S_t=s] - \mathbb{E}_{A_t\sim \hat \mu}[\gamma^2c_t^2 \mathbb{E}_{S_{t+1} \sim p}[\epsilon(S_{t+1}) \mid S_t, A_t] \mid S_t=s]\\ 
        & + \mathbb{E}_{A_t \sim \mu}[\rho_t^2 \nu_\pi(S_t, A_t) \mid S_t=s] + \mathbb{E}_{A_t \sim \mu}[\rho^2_tq^2_\pi(S_t, A_t) \mid S_t=s] - v^2_\pi(s)
    \end{split}   \\
    \begin{split}
        &= \mathbb{E}_{A_t \sim \hat \mu}[\gamma^2 c_t^2\mathbb{E}_{S_{t+1}\sim p}[\mathbb{E}_{A_{t+1} \sim \pi}[\tilde q({S_{t+1}, A_{t+1})} \mid S_{t+1}, S_t, A_t] \mid S_t, A_t]  + \rho_t^2 \nu_\pi(S_t, A_t)  + \rho^2_tq^2_\pi(S_t, A_t) \mid S_t=s]\\ 
        &-v^2_\pi(s) - \mathbb{E}_{A_t\sim \hat \mu}[\gamma^2c_t^2 \mathbb{E}_{A_t\sim \hat \mu}[\gamma^2c_t^2 \mathbb{E}_{S_{t+1} \sim p}[\epsilon(S_{t+1}) \mid S_t, A_t] \mid S_t=s]
    \end{split} \tag*{(by Lemma \ref{lem:variancequality})} \\
    &= \mathbb{E}_{A_t \sim \hat \mu}[\rho_t^2(\tilde q(S_t, A_t) + v^2_\pi(S_t)) \mid S_t=s] -v^2_\pi(s) - \mathbb{E}_{A_t\sim \hat \mu}[\gamma^2c_t^2 \mathbb{E}_{S_{t+1} \sim p}[\epsilon(S_{t+1}) \mid S_t, A_t] \mid S_t=s] \tag*{(by (\ref{eq:qtildedef}) and \text{$\rho_t = c_t$})}\\
    &= \mathbb{E}_{A_t \sim \hat \mu}[\rho_t^2\hat q(S_t, A_t) \mid S_t=s] -v^2_\pi(s) - \mathbb{E}_{A_t\sim \hat \mu}[\gamma^2c_t^2 \mathbb{E}_{S_{t+1} \sim p}[\epsilon(S_{t+1})] \tag*{(by Lemma \ref{lem:qhattildedef})}\\
    \begin{split}
        & = \mathbb{V}_{A_t \sim \mu}(\rho_t \textstyle{\sqrt{\hat q_\pi(S_t, A_t)}} \mid S_t) + \left(\mathbb{E}_{A_t \sim \mu_t}[\rho_t \textstyle{\sqrt{\hat q_\pi(S_t, A_t)}}\mid S_t] \right)^2\\
        &- v^2_\pi(s) - \mathbb{E}_{A_t\sim \hat \mu}[\gamma^2c_t^2 \mathbb{E}_{S_{t+1} \sim p}[\epsilon(S_{t+1}) \mid S_t, A_t] \mid S_t=s]
    \end{split} \tag*{(definition of variance and non-negativity)}\\
    \begin{split}
        & = \mathbb{V}_{A_t \sim \mu}(\rho_t \textstyle{\sqrt{\hat q_\pi(S_t, A_t)}} \mid S_t) + \left(\mathbb{E}_{A_t \sim \pi}[ \textstyle{\sqrt{\hat q_\pi(S_t, A_t)}}\mid S_t] \right)^2 \\
        &- v^2_\pi(s) - \mathbb{E}_{A_t\sim \hat \mu}[\gamma^2c_t^2 \mathbb{E}_{S_{t+1} \sim p}[\epsilon(S_{t+1}) \mid S_t, A_t] \mid S_t=s]
    \end{split} \tag*{(change of measure)}\\
    & =\left(\mathbb{E}_{A_t \sim \pi}[ \textstyle{\sqrt{\hat q_\pi(S_t, A_t)}}\mid S_t] \right)^2 - v^2_\pi(s) - \mathbb{E}_{A_t\sim \hat \mu}[\gamma^2c_t^2 \mathbb{E}_{S_{t+1} \sim p}[\epsilon(S_{t+1}) \mid S_t, A_t] \mid S_t=s] \tag*{(definition of \text{$\hat \mu$} gives zero variance Lemma \ref{lem:zero})}\\
    \begin{split}
    &= \mathbb{E}_{A_t \sim \pi}[\hat q_\pi(S_t, A_t) \mid S_t] - v^2_\pi(s) + \left(\mathbb{E}_{A_t \sim \pi}[ \textstyle{\sqrt{\hat q_\pi(S_t, A_t)}}\mid S_t] \right)^2\\
         &- \mathbb{E}_{A_t \sim \pi}[\hat q_\pi(S_t, A_t) \mid S_t] - \mathbb{E}_{A_t\sim \hat \mu}[\gamma^2c_t^2 \mathbb{E}_{S_{t+1} \sim p}[\epsilon(S_{t+1}) \mid S_t, A_t] \mid S_t=s]
    \end{split}\\
    \begin{split}
        & = \mathbb{V}_{\pi}[G^{\text{TIS}, \lambda}_t \mid S_t] + \left(\mathbb{E}_{A_t \sim \pi}[ \textstyle{\sqrt{\hat q_\pi(S_t, A_t)}}\mid S_t] \right)^2 \\
         &- \mathbb{E}_{A_t \sim \pi}[\hat q_\pi(S_t, A_t) \mid S_t] - \mathbb{E}_{A_t\sim \hat \mu}[\gamma^2c_t^2 \mathbb{E}_{S_{t+1} \sim p}[\epsilon(S_{t+1}) \mid S_t, A_t] \mid S_t=s]
    \end{split} \tag*{(from Lemma \ref{lem:variancequality} and Lemma \ref{lem:qhattildedef})}\\
        &= J_\pi(s) - \left(c(s) + \mathcal{K}_{\hat\mu} \epsilon(s) \right) \tag*{(by definition and \text{$\rho_t = c_t$})}\\
        & = J_\pi(s) - \epsilon(s)  
\end{align*}
Thus, $J_\pi - \epsilon$ is a fixed point of $\mathcal{J}_{\hat \mu}$, and by uniqueness of fixed points (and Lemma \ref{lem:variancebellman}),
\begin{equation*}
    J_{\hat \mu}(s) = J_\pi(s) - \epsilon(s)
\end{equation*}
and in particular,
\begin{equation*}
    \mathbb{V}_{\hat \mu}\left(G_t^{\text{TIS}, \lambda} \mid S_t =s \right) =  \mathbb{V}_{\pi}\left(G_t^{\text{TIS}, \lambda} \mid S_t = s\right) - \epsilon(s) \leq \mathbb{V}_{\pi}\left(G_t^{\text{TIS}, \lambda} \mid S_t = s\right) - \epsilon(s)
\end{equation*}
With equality only when $\epsilon(s)=0$.
\end{proof}
\subsection{Proof of Theorem \ref{thm:qhat}}
\label{sec:thmqhat}
\begin{restatedthm}{\ref{thm:qhat} (restated)}
Let,
\begin{equation*}
    \hat r_{\pi}(s,a) := 2 r(s, a) q_{\pi}(s, a) - r^2(s,a)
\end{equation*}
Then for any $\gamma \in [0, 1)$,
\begin{equation*}
    \hat q_{\pi}(s, a) := \hat r_{\pi}(s, a) + \gamma^2\mathbb{E}_{A_t\sim\pi, S_{t+1}\sim p}[\hat q_{\pi}(S_{t+1}, A_t)]
\end{equation*} 
\end{restatedthm}
\begin{proof}
    \begin{align*}
        &\hat q_{\pi}(s, a)\\
        & = \tilde q_{\pi}(s, a) + v^2_\pi(s) \tag*{(by Lemma \ref{lem:qhattildedef})}\\
        & = \tilde r_\pi(s, a) + v^2_\pi(s) + \gamma^2\mathbb{E}_{s'\sim p, a' \sim \pi}[\tilde q_\pi(s', a')] \tag*{(by (\ref{eq:qtildedef}))}\\
        & = \tilde r_\pi(s, a) + v^2_\pi(s) + \gamma^2\mathbb{E}_{s'\sim p, a' \sim \pi}[\tilde q_\pi(s', a') + v_\pi^2(s') - v^2_\pi(s')] \\
        & = \tilde r_\pi(s, a) + v^2_\pi(s) + \gamma^2\mathbb{E}_{s'\sim p, a' \sim \pi}[\hat q_\pi(s', a') - v^2_\pi(s')] \tag*{(by Lemma \ref{lem:qhattildedef})}\\
        & = \nu_\pi(s, a) + q^2_\pi(s, a) - \gamma^2\mathbb{E}_{s' \sim p} [v_\pi^2(s')] + \gamma^2\mathbb{E}_{s'\sim p, a' \sim \pi}[\hat q_\pi(s', a')] \tag*{(by (\ref{eq:rtildedef}))}\\
        & = -(\mathbb{E}_{s' \sim p}[\gamma v_\pi(s')])^2 + q^2_\pi(s, a) + \gamma^2\mathbb{E}_{s'\sim p, a' \sim \pi}[\hat q_\pi(s', a')] \tag*{(by (\ref{eq:defnu}))}\\
        & = -(q_\pi(s, a) - r(s, a))^2 + q^2_\pi(s, a) + \gamma^2\mathbb{E}_{s'\sim p, a' \sim \pi}[\hat q_\pi(s', a')]\\
        & = 2r(s, a)q_\pi(s, a) - r^2(s, a) + \gamma^2\mathbb{E}_{s'\sim p, a' \sim \pi}[\hat q_\pi(s', a')]\\
        & = \hat r_\pi(s, a) + \gamma^2\mathbb{E}_{s'\sim p, a' \sim \pi}[\hat q_\pi(s', a')] \tag*{(by (\ref{eq:rhatdef}))}
    \end{align*}
\end{proof}

\subsection{Proof of Theorem \ref{thm:policymatching}}
\label{sec:thmpolicymatching}

\begin{restatedthm}{\ref{thm:policymatching} (restated)}
If the parameterized behaviour policy 
\begin{equation*}
    \mu_{\xi}(a \vert s) \propto \pi_{\theta}(a \vert s) \textstyle{\sqrt{\hat Q_{\hat \zeta}(s, a)}}
\end{equation*} (i.e., matches (\ref{eq:propto})), then the loss $L^{\text{cont}}_{\mu}(\xi)$ is minimized. 
\end{restatedthm}

\noindent First we add some technical clarity, let $\tilde \nu(s, a) = \pi_{\theta}(a \vert s) \sqrt{\hat Q_{\hat \zeta}(s, a)}$ (assuming non-negativity of $\hat Q_{\hat \zeta}(s, a)$) and let $Z(s) = \int_a \tilde \nu(s, a)$. Then,
\begin{equation*}
    L_{\mu}^{\text{cont}}(\xi) = \hat{\mathbb{E}}_t \left[ \ln \mu_{\xi}(A_t \vert S_t) - \ln \pi_{\theta}(A_t \vert S_t) - \\ \frac{1}{2} \ln \hat Q_{\hat \zeta}(S_t, A_t) \right] 
\end{equation*}
is minimized by,
\begin{equation*}
    \mu^*_{\xi}(a \vert s) = \frac{\pi_{\theta}(a \vert s) \textstyle{\sqrt{\hat Q_{\hat \zeta}(s, a)}}}{Z(s)} \propto \pi_{\theta}(a \vert s) \textstyle{\sqrt{\hat Q_{\hat \zeta}(s, a)}}
\end{equation*}
\begin{proof}
Now writing the loss conditioning on $S_t = s$ and in terms of $\nu$,
\begin{equation*}
    \mathcal{L}(s ; \mu) = \mathbb{E}_{A_t \sim \mu}[\ln \mu(A_t \vert S_t) - \ln \tilde \nu(S_t, A_t) \mid S_t = s]
\end{equation*}
Letting $\bar \nu(a \vert s) := \tilde \nu(s, a)/Z(s)$ be the normalized density. Then,
\begin{equation*}
    \mathcal{L}(s ; \mu) = \mathbb{E}_{A_t \sim \mu}[\ln \mu(A_t \vert S_t) - \ln\bar \nu(S_t, A_t) \mid S_t = s] - \ln Z(s) = D_{\text{KL}}(\mu(\cdot \vert s) \Vert \bar \nu(\cdot \vert s)) - \ln Z(s)
\end{equation*}
By Gibbs' inequality $D_{\text{KL}}(\mu(\cdot \vert s) \Vert \bar \nu(\cdot \vert s)) \geq 0$ with equality if and only if $\mu(\cdot \vert s)= \bar \nu(\cdot \vert s)$. Since $Z(s)$ does not depend on $\mu$ then the conditional loss $\mathcal{L}(s ; \mu)$ is minimized exactly at $\mu^*(\cdot \vert s)=\bar \nu(\cdot \vert s)$. Averaging over the empirical sample (i.e., $\hat{\mathbb{E}}_t$) simply reweights the states and adds constants (independent of $\mu_{\xi}$) to the overall loss $L_{\mu}^{\text{cont}}(\xi)$, thus the overall loss is minimized by,
\begin{equation*}
    \mu^*_{\xi}(a \vert s) = \bar \nu(a \vert s) = \frac{\tilde \nu(s \vert a)}{Z(s)} = \frac{\pi_{\theta}(a \vert s) \textstyle{\sqrt{\hat Q_{\hat \zeta}(s, a)}}}{\int_a \pi_{\theta}(a \vert s) \textstyle{\sqrt{\hat Q_{\hat \zeta}(s, a)}}} \propto \pi_{\theta}(a \vert s) \textstyle{\sqrt{\hat Q_{\hat \zeta}(s, a)}}
\end{equation*}
\end{proof}
\clearpage
\section{Ablation Studies}
\label{sec:ablationstudies}

We provide here a series of ablation studies. In particular, we look at how removing $\hat Q_{\hat \zeta}$ from the loss (\ref{eq:mucontloss}) affects the convergence properties of BPO for the MuJoCo experiments. Furthermore, we investigate how the absence of \emph{symlog targets} may affect training stability and convergence.

\paragraph{No $\hat Q_{\hat \zeta}$ in (\ref{eq:mucontloss}).} Since $\hat Q_{\hat \zeta}$ is an integral part of designing the behaviour policy $\mu$, we would expect to see worse performance. Without $Q_{\hat \zeta}$ the behaviour policy $\mu$ simply matches the target policy $\pi$, without biasing it towards higher valued actions and failing to reduce the variance of the estimated returns. For collecting these results, we simply picked the best performing configuration of BPO and modified the loss function, that is,
\begin{equation*}
    {L'_{\mu}}(\xi) = \hat{\mathbb{E}}_t \bigg[ \ln \mu_{\xi}(A_t \vert S_t) - \ln \pi_{\theta}(A_t \vert S_t) \bigg]
\end{equation*}
The results for this ablation are presented in Figure \ref{fig:ablation1}.

\begin{figure*}[ht]
\centering
\includegraphics[width=0.98\linewidth]{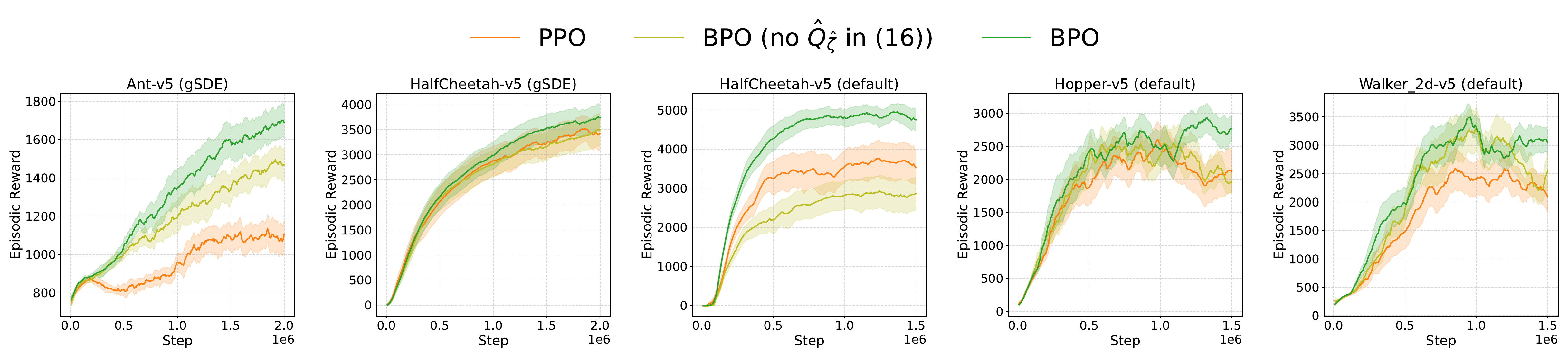}
\caption{MuJoCo mean returns (10 eval episodes) for BPO and BPO without $\hat Q_{\hat \zeta}$ in (\ref{eq:mucontloss}).}
\label{fig:ablation1}
\end{figure*}

\noindent In all cases we see that removing $\hat Q_{\hat \zeta}$ from the loss function results in worse performance, indicating that including $\hat Q_{\hat \zeta}$ in the loss (\ref{eq:mucontloss}) plays an important role in optimizing the behaviour policy $\mu$ so that it collects lower variance returns estimates. In some cases BPO without $\hat Q_{\hat \zeta}$, performs better than PPO, this can possibly be explained by the fact that the behaviour policy $\mu$ in this instance, is a essentially slower moving copy $\pi$, explosive updates to $\pi$ might not affect $\mu$ immediately at which point new data can be collected to stabilize learning. The cases that BPO without $\hat Q_{\hat \zeta}$, performs worse than PPO, can be explained with similar reasoning, perhaps the slower moving behaviour policy $\mu$ biases the data collected (due to truncation) resulting in slower learning and sub-optimal convergence. A more thorough investigation of these intuitions is left as future work.

\paragraph{No \emph{symlog targets}.} As discussed in the main text, \emph{symlog targets} are important for the stability of the two networks $Q_{\zeta}$ and $\hat Q_{\hat \zeta}$. We often find that the relative magnitudes between the predictions from $Q_{\zeta}$ and $\hat Q_{\hat \zeta}$ are large, \emph{symlog targets} help the networks quickly move their predictions to large magnitudes, the \emph{symexp} transform is used to preserve magnitudes for example in calculating the $\hat r_\pi$ values (c.f., \ref{eq:rhatdef}) and in the loss functions for the behaviour policy $\mu$, (c.f., (\ref{eq:mudiscloss}), (\ref{eq:mucontloss})). We verify this claim with a straightforward ablation study. Again, we simply pick the best performing configuration of BPO, turn \emph{symlog targets} off  and keep everything else intact. The results are presented in Figure \ref{fig:ablation2}.

\begin{figure*}[ht]
\centering
\includegraphics[width=0.98\linewidth]{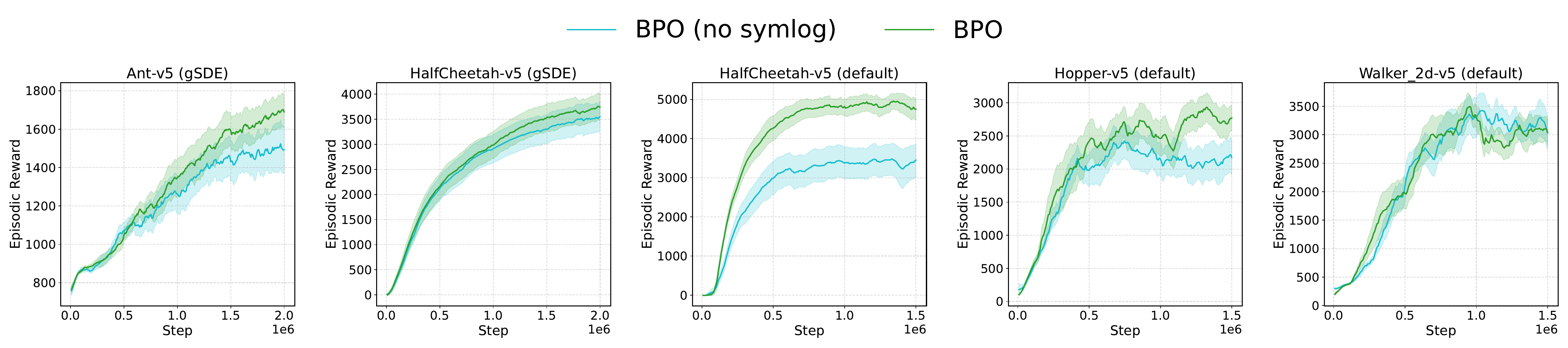}
\caption{MuJoCo mean returns (10 eval episodes) for BPO and BPO without \emph{symlog}.}
\label{fig:ablation2}
\end{figure*}

\noindent We see that in all cases, training the networks $Q_{\zeta}$ and $\hat Q_{\hat \zeta}$ without \emph{symlog targets} hurts performance, thus justifying its use, and providing some good evidence for our intuition on why this works. Further ablations, particular of those features mentioned in Appendix \ref{sec:additionalimplementation} are left as future work. 

\section{Implementation Details}

We summarize here the important implementation details, including, training and hyperparameter details. We note that the MuJoCo \cite{todorov2012mujoco} environments are provided by Gymnasium \cite{towers2025gymnasium}, see \url{https://gymnasium.farama.org/}. We also implemented the \texttt{ShortCorridor} grid world using the same Gymnasium interface, our implementation was inspired by the following re-implementation of the original experiment from \cite{sutton2018reinforcement}, see \url{https://github.com/ShangtongZhang/reinforcement-learning-an-introduction/blob/master/chapter13/short_corridor.py}. 

\paragraph{Access to code.} We provide independent implementations of all of the algorithms (REINFORCE, BPO and PPO) can be found via \url{https://github.com/sacktock/BPO}. The algorithms are implemented in \emph{Python} with JAX \cite{jax2018github}, which supports vectorized computation on GPU and JIT compilation for further optimization. For installation instructions please refer to the \texttt{README.md} file provided. For completeness we continue by providing listings for each instantiation of BPO, REINFORCE and PPO. 

\subsection{Listings}
\label{sec:listings}

\begin{listing}[ht]
    \caption{PPO update steps}
    \begin{algorithmic}[1]
        \State Generate $S_0A_0R_1\ldots S_{T-1}A_{T-1}R_{T-1}$ with $\pi$
        \State \emph{// PPO update steps}
        \For{$\text{epoch} = 0,1, \ldots \text{n\_epochs}$}
        \State Optimize $\pi_{\theta}$ with the clipped surrogate objective (\ref{eq:clippedobjective})
        \State Optimzie $V_{\omega}$ with (\ref{eq:valueloss})
        \EndFor
    \end{algorithmic}
\end{listing}

\begin{listing}[ht]
    \caption{REINFORCE update steps}
    \begin{algorithmic}[1]
        \State Generate $S_0A_0R_1\ldots S_{T-1}A_{T-1}R_{T-1}$ with $\pi$
        \State \emph{// REINFORCE update steps}
        \For{$t = 0, \ldots, T-1$}
        
        \State $G_{t}\gets \sum_{k=t}^T R_{k+1}$
        \State $\theta \gets \theta + \alpha G_{t} \nabla \ln \pi_{\theta}(A_t \vert S_t)$ (policy gradient)
        \EndFor
    \end{algorithmic}
\end{listing}

\begin{listing}[ht]
\caption{Example instantiation of BPO with the REINFORCE update steps}
    \label{alg:bporeinforce}
\begin{algorithmic}[1]
\State \textbf{Initialize} $\theta$, $\omega$, $\zeta$, $\hat \zeta$, $\xi$ and empty replay buffer $D$
\For{$\text{phase}= 0, 1, \ldots$}
\State Generate $S_0A_0R_1\ldots S_{T-1}A_{T-1}R_{T-1}$ with $\mu$
\State Add experience tuples $(s_i, a_i, r_i, s'_{i})$ to $D$
\State \emph{// REINFORCE update steps}
\For{$t = 0, \ldots, T-1$}
\State $G^{\text{TIS},\lambda}_{t}\gets \sum_{k=t}^T (\prod^{k-1}_{i=t} c_i) \rho_k\gamma^{k-t} R_{k+1}$ (with $\lambda=1.0$ and no value estimates)
\State $\theta \gets \theta + \alpha G^{\text{TIS},\lambda}_{t} \nabla \ln \pi_{\theta}(A_t \vert S_t)$ (policy gradient)
\EndFor
\State \emph{// Auxilliary steps}
\For{$\text{epoch} = 0,1, \ldots \text{n\_qvf\_epochs}$}
 \State Optimize $Q_{\zeta}$ with FQE (c.f., (\ref{eq:qlossfn}))
\EndFor
\State Compute $\hat r_t$ (c.f., (\ref{eq:rhatdef}))
\For{$\text{epoch} = 0,1, \ldots \text{n\_qvf\_epochs}$}
    \State Optimize $\hat Q_{\hat\zeta}$ with FQE (c.f., (\ref{eq:qhatlossfn}))
\EndFor
\For{$\text{epoch} = 0,1, \ldots \text{n\_mu\_epochs}$}
    \State Optimize behaviour policy $\mu_{\xi}$ (c.f., (\ref{eq:mudiscloss}) or (\ref{eq:mucontloss}))
\EndFor
\EndFor
\end{algorithmic}
\end{listing}
\clearpage

\begin{listing}[ht]
\caption{Example instantiation of BPO with the PPO update steps}
    \label{alg:bpoppo}
\begin{algorithmic}[1]
\State \textbf{Initialize} $\theta$, $\omega$, $\zeta$, $\hat \zeta$, $\xi$ and empty replay buffer $D$
\For{$\text{phase}= 0, 1, \ldots$}
\State Generate $S_0A_0R_1\ldots S_{T-1}A_{T-1}R_{T-1}$ with $\mu$
\State Add experience tuples $(s_i, a_i, r_i, s'_{i})$ to $D$
\State \emph{// PPO update steps}
\For{$\text{epoch} = 0,1, \ldots \text{n\_epochs}$}
        \State Optimize $\pi_{\theta}$ with the clipped surrogate objective (\ref{eq:clippedobjective})
        \State Optimzie $V_{\omega}$ with (\ref{eq:valueloss})
        \EndFor
\State \emph{// Auxilliary steps}
\For{$\text{epoch} = 0,1, \ldots \text{n\_qvf\_epochs}$}
 \State Optimize $Q_{\zeta}$ with FQE (c.f., (\ref{eq:qlossfn}))
\EndFor
\State Compute $\hat r_t$ (c.f., (\ref{eq:rhatdef}))
\For{$\text{epoch} = 0,1, \ldots \text{n\_qvf\_epochs}$}
    \State Optimize $\hat Q_{\hat\zeta}$ with FQE (c.f., (\ref{eq:qhatlossfn}))
\EndFor
\For{$\text{epoch} = 0,1, \ldots \text{n\_mu\_epochs}$}
    \State Optimize behaviour policy $\mu_{\xi}$ (c.f., (\ref{eq:mudiscloss}) or (\ref{eq:mucontloss}))
\EndFor
\EndFor
\end{algorithmic}
\end{listing}

\clearpage
\subsection{Additional Implementation Details}
\label{sec:additionalimplementation}

We summarize the additional tricks that we use to stabilize learning and prevent overestimation issues for the two Q-networks $Q_{\xi}$ and $\hat Q_{\hat \xi}$:

\begin{enumerate}
    \item \emph{Polyak averaging} ($\tau$): To prevent overestimation, we also regularize the predictions of $Q_{\zeta}$ and $\hat Q_{\hat \zeta}$ towards slower moving target networks updated with Polyak averaging \cite{polyak1992acceleration}. 
    \item \texttt{Layer norm} and \texttt{Zero norm init}: At the architecture level, we use layer-normalization \cite{ba2016layer} for more stable learning and we initialize the weights of the final layer of the networks to zeros to prevent early overestimation issues and explosions often associated with \emph{symlog targets} \cite{hafner2024mastering}.
    \item \texttt{Weight TD update}: At the batch level, the losses in (\ref{eq:qlossfn}) and (\ref{eq:qhatlossfn}) are element-wise weighted with the corresponding IS ratios, so that out-of-distribution or old samples from the replay buffer are down weighted in the update. These weights are then normalized at the batch level to maintain consistent gradient magnitudes across batch updates. 
    \item \texttt{Clip targets}: To prevent any further possible explosions associated with the $\text{symexp}$ transform we clip all predictions and targets with the upper bounds $r_{\max}/(1 - \gamma)$ and $\hat r_{\max}/(1 - \gamma^2)$, where  $r_{\max}$ and $\hat r_{\max}$ are estimated during training.  
\end{enumerate}

\paragraph{Training Details.} All sets of experiments are run on Ubuntu 22.04, with access to 2 NVIDIA Tesla A40 (48GB RAM) GPU and a 24-core/48 thread Intel Xeon CPU each with 16GB of additional CPU RAM. For the REINFORCE experiments, results (for 100 independent runs) can be collected in a matter of hours (usually no more than 2). For the MuJoCo experiments, results (for 10 independent runs) can be collected in no more than 36 hours with these compute resources. 

\paragraph{Wall-clock comparison.} Due to the overhead with training 3 additional networks (i.e., $\mu$, $Q_{\xi}$, $\hat Q_{\hat \xi}$) in parallel with the actor $\pi$ and value network $V$. BPO takes longer to complete one training run compared to the underlying PPO. For the hardware summarized above, one run for PPO can take approximately 1.6 hours compared to 3.5 hours for BPO, a slow down by a factor of at least 2 time. We maintain that this is reasonable, particularly if BPO converged to a policy with high overall returns. 


\begin{table}[ht]
  \centering
  \begin{tabular}{lcc}
    \toprule
       Name & Symbol & Value  \\
    \midrule
    \multicolumn{3}{c}{REINFORCE (shared)}\\
    \midrule
    Learning rate max & $\alpha_{\text{max}}$ & $0.1$\\
    Learning rate min & $\alpha_{\text{min}}$ & $0.01$\\
    Schedule & - & exponential decay\\
    Discount factor & $\gamma$ & $0.99$\\
    \midrule
    \multicolumn{3}{c}{BPO}\\
    \midrule
    ``TD error'' truncation & $\bar \rho$ & 1.5 \\
    ``Trace cutting'' & $\bar c$ & 1.0 \\
    (\texttt{traj}) clipping & - & \texttt{False} \\
    Replay size & $\vert D \vert$ & 1024\\
    Batch size & $\vert B \vert$ & 256\\
    $\mu$ training epochs & \texttt{n\_mu\_epochs} & 1\\
    $\mu$ optimizer & - & adam (max\_grad=0.5, $\epsilon=10^{-5}$)\\
    $\mu$ learning rate & $\alpha_\mu$ & 0.001 \\
    $\mu$ network & - & 2-layers; 64 units; \texttt{ReLU}\\
    $(Q_{\xi}$/$\hat Q_{\hat \xi})$ training epochs & \texttt{n\_qvf\_epochs} & 1\\
    $(Q_{\xi}$/$\hat Q_{\hat \xi})$ optimizer & - & adam ($\epsilon= 10^{-5}$) \\
    $(Q_{\xi}$/$\hat Q_{\hat \xi})$ learning rate & $\alpha_Q$ & 0.001\\
    $(Q_{\xi}$/$\hat Q_{\hat \xi})$ network & - & 2-layers; 64 units; \texttt{ReLU}\\
    \emph{Symlog targets} & - & \texttt{True}\\
    \emph{Symlog} regularization & - & 1.0\\
    \emph{Polyak} tau & $\tau$ & 0.02\\
    \texttt{Weight TD update} & - & \texttt{True}\\
    \texttt{Clip targets} & - & \texttt{True}\\
    \texttt{Zero norm init} & - & \texttt{True}\\
    \texttt{Layer norm} & - & \texttt{True}\\
    \bottomrule
  \end{tabular}
  \caption{Hyperparameter details for Figure \ref{fig:reinforce}: REINFORCE and BPO (REINFORCE).}
  \label{tab:reinforcehyperparams}
\end{table}

\begin{table}[t]
  \centering
  \begin{tabular}{lcc}
    \toprule
       Name & Symbol & Value  \\
    \midrule
    \multicolumn{3}{c}{PPO (shared)}\\
    \midrule
    Num. steps & $T$ & 512\\
    Num. envs. & - & 16\\
    Batch size & - & 128\\
    $\pi$ training epochs & \texttt{n\_epochs} & 20\\
    Discount factor & $\gamma$ & 0.99\\
    (GAE) lambda & $\lambda$ & 0.9\\
    Clip range & $\epsilon$ & 0.4\\
    Normalize & - & observations, advantages\\
    Ent. coef. & $\beta_{\text{ent}}$ & 0.0\\
    Value coef. & $\beta_{\text{value}}$ & 0.5\\
    \texttt{gSDE} & - & \texttt{True}\\
    \texttt{gSDE} Sample freq. & - & 4\\
    $\pi$ optimizer & - & adam (max\_grad=0.5, $\epsilon=10^{-5}$)\\
    $\pi$ learning rate & - & $3\times 10^{-5}$\\
    Initial log std  & $\log \sigma$ & -1.0\\
    $\pi$ network & - & 2-layers; 256 units; \texttt{ReLU}\\
    Tanh squashed $\pi$ - & \texttt{False}\\
    $V$ optimizer & - & adam (max\_grad=0.5, $\epsilon=10^{-5}$)\\
    $V$ learning rate & - & $3\times 10^{-5}$\\
    $V$ network & - & 2-layers; 256 units; \texttt{ReLU}\\
    \midrule
    \multicolumn{3}{c}{BPO}\\
    \midrule
    ``TD error'' truncation & $\bar \rho$ & $\{ 1.0, 1.5\}$ \\
    ``Trace cutting'' & $\bar c$ & $\{ 1.0, 1.5\}$  \\
    (\texttt{traj}) clipping & - & $\{ \texttt{True}, \texttt{False}\}$ \\
    Replay size & $\vert D \vert$ & 8192\\
    Batch size & $\vert B \vert$ & 128\\
    $\mu$ training epochs & \texttt{n\_mu\_epochs} & 40\\
    $\mu$ optimizer & - & adam (max\_grad=0.5, $\epsilon=10^{-5}$)\\
    $\mu$ learning rate & $\alpha_\mu$ & $3\times 10^{-5}$ \\
    $\mu$ network & - & 2-layers; 256 units; \texttt{ReLU}\\
    $(Q_{\xi}$/$\hat Q_{\hat \xi})$ training epochs & \texttt{n\_qvf\_epochs} & 40\\
    $(Q_{\xi}$/$\hat Q_{\hat \xi})$ optimizer & - & adam (max\_grad=0.5, $\epsilon=10^{-5}$) \\
    $(Q_{\xi}$/$\hat Q_{\hat \xi})$ learning rate & $\alpha_Q$ & $3\times 10^{-5}$\\
    $(Q_{\xi}$/$\hat Q_{\hat \xi})$ network & - & 2-layers; 256 units; \texttt{ReLU}\\
    \emph{Symlog targets} & - & \texttt{True}\\
    \emph{Symlog} regularization & - & 1.0\\
    \emph{Polyak} tau & $\tau$ & 0.02\\
    \texttt{Weight TD update} & - & \texttt{False}\\
    \texttt{Clip targets} & - & \texttt{False}\\
    \texttt{Clip actions} & - & \texttt{True}\\
    \texttt{Zero norm init} & - & \texttt{True}\\
    \texttt{Layer norm} & - & \texttt{True}\\
    \bottomrule
  \end{tabular}
  \caption{Hyperparameter details for MuJoCo experiments (\texttt{gSDE}): PPO and BPO (PPO).}
  \label{tab:ppogsdehyperparameters}
\end{table}

\begin{table}[t]
  \centering
  \begin{tabular}{lcc}
    \toprule
       Name & Symbol & Value  \\
    \midrule
    \multicolumn{3}{c}{PPO (shared)}\\
    \midrule
    Num. steps & $T$ & 2048\\
    Num. envs. & - & 1\\
    Batch size & - & 64\\
    $\pi$ training epochs & \texttt{n\_epochs} & 10\\
    Discount factor & $\gamma$ & 0.99\\
    (GAE) lambda & $\lambda$ & 0.95\\
    Clip range & $\epsilon$ & 0.2\\
    Normalize & - & observations, advantages\\
    Ent. coef. & $\beta_{\text{ent}}$ & 0.001\\
    Value coef. & $\beta_{\text{value}}$ & 0.5\\
    \texttt{gSDE} & - & \texttt{False}\\
    $\pi$ optimizer & - & adam (max\_grad=0.5, $\epsilon=10^{-5}$)\\
    $\pi$ learning rate & - & $3\times 10^{-4}$\\
    Initial log std  & $\log \sigma$ & -1.0\\
    $\pi$ network & - & 2-layers; 64 units; \texttt{ReLU}\\
    Tanh squashed $\pi$ - & \texttt{False}\\
    $V$ optimizer & - & adam (max\_grad=0.5, $\epsilon=10^{-5}$)\\
    $V$ learning rate & - & $3\times 10^{-4}$\\
    $V$ network & - & 2-layers; 64 units; \texttt{ReLU}\\
    \midrule
    \multicolumn{3}{c}{BPO}\\
    \midrule
    ``TD error'' truncation & $\bar \rho$ & $\{ 1.0, 1.5\}$  \\
    ``Trace clipping'' & $\bar c$ & $\{ 1.0, 1.5\}$  \\
    (\texttt{traj}) clipping & - & $\{ \texttt{True}, \texttt{False}\}$ \\
    Replay size & $\vert D \vert$ & 8192\\
    Batch size & $\vert B \vert$ & 128 / 256 ($\mu$ / $Q$)\\
    $\mu$ training epochs & \texttt{n\_mu\_epochs} & 20\\
    $\mu$ optimizer & - & adam (max\_grad=0.5, $\epsilon=10^{-5}$)\\
    $\mu$ learning rate & $\alpha_\mu$ & $3\times 10^{-4}$ \\
    $\mu$ network & - & 2-layers; 64 units; \texttt{ReLU}\\
    $(Q_{\xi}$/$\hat Q_{\hat \xi})$ training epochs & \texttt{n\_vf\_epochs} & 20\\
    $(Q_{\xi}$/$\hat Q_{\hat \xi})$ optimizer & - & adam (max\_grad=0.5, $\epsilon=10^{-5}$) \\
    $(Q_{\xi}$/$\hat Q_{\hat \xi})$ learning rate & $\alpha_Q$ & $3\times 10^{-4}$\\
    $(Q_{\xi}$/$\hat Q_{\hat \xi})$ network & - & 2-layers; 64 units; \texttt{ReLU}\\
    \emph{Symlog targets} & - & \texttt{True}\\
    \emph{Symlog} regularization & - & 1.0\\
    \emph{Polyak} tau & $\tau$ & 0.02\\
    \texttt{Weight TD update} & - & \texttt{True}\\
    \texttt{Clip targets} & - & \texttt{True}\\
    \texttt{Clip actions} & - & \texttt{True}\\
    \texttt{Zero norm init} & - & \texttt{True}\\
    \texttt{Layer norm} & - & \texttt{True}\\
    \bottomrule
  \end{tabular}
  \caption{Hyperparameter details for MuJoCo experiments (\texttt{default}): PPO and BPO (PPO).}
  \label{tab:ppodefaulthyperparameters}
\end{table}

\begin{table}[t]
  \centering
  \begin{tabular}{lcc}
    \toprule
       Name & Symbol & Value  \\
       \midrule
    \multicolumn{3}{c}{Ant-v5 (\texttt{gSDE})}\\
    \midrule
    Normalize & - & advantages\\
    \midrule
    \multicolumn{3}{c}{HalfCheetah-v5 (\texttt{gSDE})}\\
    \midrule
    initial log std  & $\log \sigma$ & -2.0\\
    \bottomrule
  \end{tabular}
  \caption{Environment specific hyperparameters for MuJoCo experiments.}
  \label{tab:envhyperparameters}
\end{table}

\end{document}